\documentclass[twoside,11pt]{article}

\usepackage{blindtext}

\usepackage[preprint]{jmlr2e}

\usepackage{lastpage}
\jmlrheading{}{}{1-\pageref{LastPage}}{}{}{}{}
\jmlrheading{}{2025}{1-\pageref{LastPage}}{11/25
}{}{21-0000}{Alberto Maria Metelli, Simone Drago, and Marco Mussi}

\ShortHeadings{Generalized Kernelized Bandits}{Metelli, Drago, and Mussi}
\firstpageno{1}

\usepackage{booktabs}       
\usepackage{makecell} 
\usepackage{amsfonts}       
\usepackage{nicefrac}       
\usepackage{microtype}      
\usepackage{xcolor}         
\usepackage{enumitem}
\usepackage[many]{tcolorbox}
\usepackage{mymacros}
\usepackage{booktabs}
\usepackage{makecell}
\usepackage{multicol,multirow}
\usepackage{algorithm2e}
\usepackage[makeroom]{cancel}
\usepackage{wrapfig}
\pgfplotsset{compat=1.18}

\newcommand{\xs}{\mathbf{x}}
\newcommand{\eL}[1][t]{\mathcal{L}_{#1}}
\newcommand{\fhat}[1][t]{\widehat{f}_{#1}}
\newcommand{\ftilde}[1][t]{\widetilde{f}_{#1}}
\newcommand{\Cset}[1][t]{\mathcal{C}_t(\delta)}
\newcommand{\gtau}{\mathfrak{g}(\tau)}
\newcommand{\Hop}[1][t]{{\widetilde{V}_{#1}}}

\newcommand{\algname}{\textnormal{\texttt{Generalized Kernelized Bandits-Upper Confidence Bounds}}\@\xspace}
\newcommand{\algnameshort}{\textnormal{\texttt{GKB-UCB}}\@\xspace}
\newcommand{\algnameeff}{\textnormal{\texttt{Tractable-Generalized Kernelized Bandits-Upper Confidence Bounds}}\@\xspace}
\newcommand{\algnameeffshort}{\textnormal{\texttt{Trac-GKB-UCB}}\@\xspace}

\renewcommand{\citet}{\cite}

\allowdisplaybreaks[4]

\begin{document}

\setlength{\abovedisplayskip}{5pt}
\setlength{\belowdisplayskip}{5pt}
\setlength{\textfloatsep}{6pt}

\title{Generalized Kernelized Bandits: A Novel Self-Normalized Bernstein-Like Dimension-Free Inequality and Regret Bounds}

\author{\name Alberto Maria Metelli \email albertomaria.metelli@polimi.it\\
       \name Simone Drago \email simone.drago@polimi.it\\
       \name Marco Mussi \email marco.mussi@polimi.it\\
       \addr Politecnico di Milano\\
       Piazza Leonardo da Vinci 32, Milan, 20133, Italy}

\editor{}

\maketitle

\begin{abstract}
We study the regret minimization problem in the novel setting of \emph{generalized kernelized bandits} (GKBs), where we optimize an unknown function $f^*$ belonging to a \emph{reproducing kernel Hilbert space} (RKHS) having access to samples generated by an \emph{exponential family} (EF) reward model whose mean is a non-linear function $\mu(f^*)$. This setting extends both \emph{kernelized bandits} (KBs) and \emph{generalized linear bandits} (GLBs), providing a \emph{unified view} of both settings. We propose an optimistic regret minimization algorithm, \algnameshort, and we explain why existing self-normalized concentration inequalities used for KBs and GLBs do not allow to provide tight regret guarantees. For this reason, we devise a novel self-normalized Bernstein-like dimension-free inequality that applies to a Hilbert space of functions with bounded norm, representing a contribution of independent interest. Based on it, we analyze \algnameshort, deriving a regret bound of order $\widetilde{O}( \gamma_T \sqrt{T/\kappa_*})$, being $T$ the learning horizon, ${\gamma}_T$ the maximal information gain, and $\kappa_*$ a term characterizing the magnitude of the expected reward non-linearity. Our result is tight in its dependence on $T$, $\gamma_T$, and $\kappa_*$ for both KBs and GLBs. Finally, we present a tractable version \algnameshort, \algnameeffshort, which attains similar regret guarantees, and we discuss its time and space complexity.
\end{abstract}

\begin{keywords}
  Generalized linear bandits, kernelized bandits, self-normalized concentration
\end{keywords}

\section{Introduction}

Extending \emph{multi-armed bandits}~\citep[MABs,][]{lai1985asymptotically} to continuous decision spaces requires introducing a structure in the expected rewards. Without such a structure, information gathered from explored decisions cannot be transferred to unexplored ones, making learning infeasible~\citep{bubeck2011x}. The most known and studied structure is the \emph{linear} one, that led to the design of \emph{linear bandits}~\citep[LBs,][]{dani2008stochastic,abbasiyadkori2011improved}. In LBs, the expected reward is modeled as the scalar product between the decision, associated with a \emph{finite-dimensional} vector $\xs_t \in \mathbb{R}^d$, and an unknown parameter vector $\vtheta^* \in \mathbb{R}^d$, \ie $\mathbb{E}[y_t|\xs_t;\bm{\theta}^*] = \xs_t^\top\bm{\theta}^*$. 
This setting strictly generalizes the finite-arm MABs~\citep{lai1985asymptotically} that can be retrieved considering as decisions the $\mathbb{R}^d$ canonical basis.
LBs, in turn, have been extended in parallel in two directions: \emph{generalized linear bandits} \citep[GLBs,][]{filippi2010parametric} and \emph{kernelized bandits} \citep[KBs,][]{srinivas2010gaussian,chowdhury2017kernelized}.
On the one hand, GLBs employ a \emph{generalized linear model}~\citep[GLM,][]{mccullagh2019generalized} to represent different reward models belonging to the \emph{exponential family} \citep[EF,][]{barndorff2014information}, including Gaussian, Bernoulli, and Poisson.\footnote{For the specific case of Bernoulli rewards, we refer to them as \emph{logistic bandits} \citep{faury2020improved}.} 
This is achieved with the use of a non-linear \emph{inverse link function} $\mu(\cdot)$, such that the expected reward is defined as $\mathbb{E}[y_t|\xs_t;\bm{\theta}^*] = \mu ( \xs_t^\top\bm{\theta}^* )$.
On the other hand, KBs focus on the optimization of an unknown expected reward function $f^*$ belonging to a \emph{reproducing kernel Hilbert space} $\mathcal{H}$ \citep[RKHS,][]{smola1998learning} induced by a known kernel function $k(\xs,\xs')$, leading to the expected reward defined as $\mathbb{E}[y_t|\xs_t;f^*] = f^*(\xs_t)$. This setting allows associating each decision $\xs_t$ with a (possibly infinite-dimensional) feature map $\phi(\xs_t) = k(\cdot,\xs_t) \in \Hs$ which generalizes the finite-dimensional vector $\vtheta^*$ of (G)LBs. KBs have often been approached using Gaussian processes~\citep[GPs,][]{rasmussen2006gaussian} but studied with a subgaussian reward model only.
It is worth noting that GLBs fall back to LBs when the identity link function $\mu = I$ is considered, and KBs fall back to LBs when a linear kernel $k(\xs,\xs')= {\xs}^\top{\xs'}$ is considered.

\begin{wrapfigure}{r}{0.485\linewidth}
\vspace{-.95cm}
\begin{center}
\resizebox{.95\linewidth}{!}{
\tikzset{every picture/.style={line width=0.75pt}}
\begin{tikzpicture}[x=0.75pt,y=0.75pt,yscale=-1,xscale=1]

\draw  [draw opacity=0][fill={vibrantGrey},
fill opacity=0.3 ] (240.2,228.77) .. controls (240.2,179.94) and (300.64,140.35) .. (375.2,140.35) .. controls (449.76,140.35) and (510.2,179.94) .. (510.2,228.77) .. controls (510.2,277.61) and (449.76,317.2) .. (375.2,317.2) .. controls (300.64,317.2) and (240.2,277.61) .. (240.2,228.77) -- cycle ;
\draw  [draw opacity=0][fill={vibrantRed},
fill opacity=0.3 ] (330.4,210.17) .. controls (330.4,182.66) and (368.39,160.35) .. (415.26,160.35) .. controls (462.12,160.35) and (500.11,182.66) .. (500.11,210.17) .. controls (500.11,237.69) and (462.12,260) .. (415.26,260) .. controls (368.39,260) and (330.4,237.69) .. (330.4,210.17) -- cycle ;
\draw  [draw opacity=0][fill={vibrantBlue},
fill opacity=0.3 ] (250.11,210.32) .. controls (250.11,182.72) and (288.16,160.35) .. (335.1,160.35) .. controls (382.04,160.35) and (420.09,182.72) .. (420.09,210.32) .. controls (420.09,237.91) and (382.04,260.29) .. (335.1,260.29) .. controls (288.16,260.29) and (250.11,237.91) .. (250.11,210.32) -- cycle ;

\draw (365,174) node [anchor=north west][inner sep=0.75pt]   [align=left] {LB};
\draw (290,174) node [anchor=north west][inner sep=0.75pt]   [align=left] {GLB};
\draw (433,174) node [anchor=north west][inner sep=0.75pt]   [align=left] {KB};
\draw (327,266) node [anchor=north west][inner sep=0.75pt]   [align=left] {GKB \ \ $\mu (f^*(\xs_t))$};
\draw (355,190) node [anchor=north west][inner sep=0.75pt]   [align=left] {$\xs_t^\top{\bm{\theta}^*}$};
\draw (262,190) node [anchor=north west][inner sep=0.75pt]   [align=left] {$\mu(\xs_t^\top{\bm{\theta}^*})$};
\draw (435,190) node [anchor=north west][inner sep=0.75pt]   [align=left] {$f^*(\xs_t)$};

\draw (276,215) node [anchor=north west][inner sep=0.75pt] [align=left] {\blue{$d < \infty $}}; 
\draw (290,230) node [anchor=north west][inner sep=0.75pt] [align=left] {\red{$\Dot{\mu} \in \mathbb{R}_{\geq 0}$}};

\draw (357,215) node [anchor=north west][inner sep=0.75pt] [align=left] {\blue{$d < \infty $}}; 
\draw (360,230) node [anchor=north west][inner sep=0.75pt] [align=left] {\blue{$\Dot{\mu} = 1 $}};

\draw (430,215) node [anchor=north west][inner sep=0.75pt] [align=left] {\red{$d = \infty $}}; 
\draw (422,230) node [anchor=north west][inner sep=0.75pt] [align=left] {\blue{$\Dot{\mu} = 1 $}};

\draw (322,288) node [anchor=north west][inner sep=0.75pt] [align=left] {\red{$d = \infty $} \ \ \red{$\Dot{\mu} \in \mathbb{R}_{\geq 0}$}};

\end{tikzpicture}}
\end{center}
\vspace{-.645cm}
\caption{Graphical representation of the settings ($f^*(\cdot)$ belong to an RKHS).}
\label{fig:venn}
\vspace{-.6cm}
\end{wrapfigure}
In this work, we propose the novel \emph{generalized kernelized bandit} (GKB) setting, which unifies GLBs and KBs (Figure~\ref{fig:venn}). This setting enables learning where the unknown expected reward function $f^*$ comes from an RKHS (as in KBs) and the samples come from an exponential family model (as in GLBs) whose mean is obtained by applying an inverse link function $\mu$ to function $f^*$, i.e., $\mathbb{E}[y_t|\xs_t;f^*] = \mu(f^*(\xs_t))$. We aim to design regret minimization algorithms for GKBs that attain sensible regret guarantees and provide tight bounds for KBs and GLBs for the specific choices of the kernel function $k$ and inverse link function $\mu$ discussed above. Specifically, we seek regret bounds that simultaneously depend on the relevant characteristics of both the kernel function $k$, namely the maximal information gain $\gamma_T$ \citep{srinivas2010gaussian}, and of the inverse link function, namely an index of the non-linearity of $\mu$ evaluated at the optimal decision $\xs^*$, i.e., $\kappa_* = {\dot{\mu}(f^*(\xs^*))}^{-1}$ \citep{faury2020improved}. 
At an intuitive level, achieving this goal requires building on top of the techniques used in GLBs and KBs. Focusing on \emph{optimistic} algorithms, as established by a large portion of the literature~\citep{abbasiyadkori2011improved,chowdhury2017kernelized,faury2020improved,lee2024unified}, a fundamental technical tool is given by \emph{self-normalized} concentration inequalities for vector-valued martingales~\citep{delapena2004self}. Unfortunately, applying the inequalities designed for either KBs or GLBs to our GKB setting is ineffective for achieving the goal defined above. 
Attempting first with the seminal inequality of \citet{abbasiyadkori2011improved}, or even its extension to Hilbert spaces \citep{abbasi2013online}, would introduce an undesirable dependence on the \emph{maximum} value of the non-lineary index of $\mu$ over the decision set, $\kappa_{\Xs} = \sup_{\xs \in \Xs}{\dot{\mu}(f^*(\xs))}^{-1} \gg \kappa_*$. This occurs because such inequalities are designed for subgaussian reward models (i.e., \emph{Hoeffding-like}) and cannot adapt to the different variances induced by the non-linearity of $\mu$, a phenomenon also present in GLBs \citep{faury2020improved}. The issue persists even when using the inequality of \citet{chowdhury2017kernelized}, which directly operates in the kernel space but still restricts to the subgaussian case. Alternatively, attempting to use the inequalities designed for GLBs, such as that of \citet{faury2020improved}, which explicitly account for the EF reward model and, thus, for the non-linearity of $\mu$ that induces different variances across decisions (i.e., \emph{Bernstein-like}), would remove the dependence on the maximum of the non-linearity index $\kappa_{\Xs}$. Nevertheless, unlike \citep{abbasiyadkori2011improved,chowdhury2017kernelized}, which are \emph{dimension-free}, such inequalities would introduce an explicit dependence on the dimension $d$ of $\vtheta^*$ (i.e., \emph{dimension-dependent}), which prevents their application to GKBs, where the dimension of the RKHS may be infinite.
These observations suggest that a new \emph{Bernstein-like dimension-free} self-normalized concentration inequality is required, which will represent a technical contribution of independent interest of this work.

\paragraph{Outline and Contributions.\protect\footnote{A preliminary version of this work, containing results substantially similar to those presented here, was presented at the 18th European Workshop on Reinforcement Learning \citep{metelli2025a}.}}
The contribution of the paper are summarized as follows:
\begin{itemize}[noitemsep, leftmargin=14pt, topsep=0pt]
    \item In Section~\ref{sec:prob}, we introduce the novel setting of \emph{generalized kernelized bandits} (GKBs), the learning problem, and the assumptions. We then characterize it based on the non-linearity index $\kappa_{*}$ and maximum non-linearity index $\kappa_{\Xs}$ of the inverse link function $\mu$, and on the kernel $k$. We show how combining a function $f^*$ belonging to the RKHS with the kernel function $k$, with maximal information gain $\gamma_T$, induces a \emph{weighted kernel} $\widetilde{k}(\cdot,\cdot;f^*)$ with a corresponding weighted maximal information gain $\widetilde{\gamma}_T(f^*)$, and we establish their relation.
    \item In Section~\ref{sec:alg}, we present \algname\ (\algnameshort), an \emph{optimistic} regret minimization algorithm composed of two steps: maximum likelihood estimation and optimistic decision selection. Both steps require solving an optimization problem over a (subset of) the RKHS.
    \item In Section~\ref{sec:challenges}, we formally discuss the limitations of the existing self-normalized concentration inequalities designed for KBs and GLBs when used for the regret analysis of \algnameshort. We then develop a novel \emph{self-normalized Bernstein-like dimension-free} inequality for martingales taking values in a Hilbert (sub)space of functions with bounded norm. This result constitutes a contribution of independent interest.\footnote{Concurrently with the present work, the preprint \citep{bae2025neural} derived an inequality similar to ours using comparable techniques, but restricted to finite-dimensional spaces. Likewise, the concurrent preprint \citep{akhavan2025bernstein} addressed our proposed open problem \citep{mussi2024open}, obtaining a bound that achieves almost all the desired properties (see also Table \ref{tab:summary_concentration_properties}.}
    \item In Section~\ref{sec:regret}, we analyze \algnameshort using a confidence set defined via our inequality and show that it achieves a regret of order $\widetilde{O}(\gamma_T \sqrt{T/\kappa_*})$. Our result matches, up to logarithmic factors, the state-of-the-art bounds for both KBs \citep{chowdhury2017kernelized} and GLBs \citep{faury2020improved,lee2024unified} in their dependence on $T$, $\gamma_T$, and $\kappa_*$.\footnote{As we discuss in Section~\ref{sec:regret}, in the GLB setting, our result exhibits a worse dependence on other problem-dependent terms when compared to algorithms specifically designed for GLBs \citep{lee2024unified}.}
    \item In Section~\ref{sec:efficient}, we introduce the \algnameeff (\algnameeffshort) algorithm, a \emph{tractable implementation} of \algnameshort that replaces optimization in the RKHS with convex programs involving a finite number of real variables. We then study the time and space complexity of \algnameeffshort and provide a regret analysis showing that, with a modified confidence set, it achieves regret guarantees of the same order $\widetilde{O}(\gamma_T \sqrt{T/\kappa_*})$  as \algnameshort.
\end{itemize}
For the logistic case, i.e., when $\mu$ is the sigmoid function, the contributions of this paper provide an answer to the open problem we posed one year ago \citep{mussi2024open}.

\section{Preliminaries}\label{sec:prelim}

\paragraph{Notation.}
Let $a,b \in \Nat$ with $a \le b$, we introduce the symbols $\dsb{a,b} \coloneqq\{a,a+1,\dots,b\}$ and $\dsb{b} \coloneqq \dsb{1,b}$. Let $d \in \Nat$, $\mathbf{I}_d$ denotes the identity matrix of order $d$ and $\mathbf{0}_d$ the column vector of all zeros of size $d$ ($d$ is omitted when clear from the context). $\mathcal{N}(\bm{\mu}, \bm{\Sigma})$ denotes the multi-variate Gaussian distribution with mean $\bm{\mu}$ and covariance matrix $\bm{\Sigma}$.

\paragraph{Reproducing Kernel Hilbert Space.}
Let $\Xs \subseteq \Reals^d$ be a decision set and $\Hs$ be a Hilbert space 
endowed with the inner product $\inner{\cdot}{\cdot}$ (and induced norm $\|\cdot\|$). $\Hs$ is a \emph{reproducing kernel Hilbert space}~\citep{smola1998learning} if ($i$) there exists a function $k : \Xs \times \Xs \rightarrow \Reals$, called \emph{kernel}, satisfying the \emph{reproducing property}, i.e., for every $f \in \Hs$, it holds that $f(\xs) = \inner{f}{k(\xs,\cdot)}$ for every $\xs \in \Xs$ and ($ii$) kernel $k$ spans $\mathcal{H}$.\footnote{This means that $\mathcal{H} = \overline{\mathrm{span}(\{k(\xs,\cdot) : \xs\in\mathcal{X}\})}$, where $\overline{\mathcal{Y}}$ denotes the completion of set $\mathcal{Y}$.} It follows that the kernel $k$ is symmetric and positive semidefinite. We define $\phi(\xs) \coloneqq k(\xs,\cdot) \in \Hs$ for every $\xs \in \Xs$, so that for all $\xs, \xs' \in \Xs$, we have $k(\xs, \xs') = \inner{\phi(\xs)}{\phi(\xs')}$.
The function $\phi$ is often referred to as the \emph{feature map} induced by the kernel $k$. We denote by $I$ the identity operator on $\Hs$. Given $f,g \in \mathcal{H}$, we denote with $fg^\top : \mathcal{H} \rightarrow \mathcal{H}$ its tensor product, i.e., $fg^\top h = \inner{g}{h}f$ for every $h \in \Hs$. Let  $A : \mathcal{H} \rightarrow \mathcal{H}$  be a bounded self-adjoint positive semidefinite linear operator and $f \in \mathcal{H}$, we use the abbreviation $\|f\|_{A}^2 = \inner{f}{Af}$. Let $A : \mathcal{H} \rightarrow \mathcal{H}$ be a bounded linear trace-class operator, we denote with $\det(I+A)$ the Fredholm determinant of $I + A$ \citep{gohberg2012traces}.
We note that for every $\xs \in \Xs$, we have that $|f(\xs)|\le \|f\| \|k(\cdot,\xs)\| = \|f\| \sqrt{k(\xs,\xs)}$. We denote with $\mathrm{dim}(\Hs)$ the dimension of the Hilbert space $\Hs$, i.e., the cardinality of an orthonomal basis, that can be infinite. 

\paragraph{Information Gain.}
Let $k$ be a kernel, let $t \in \Nat$, and let $\xs_1,\dots,\xs_{t-1} \in \Xs$ be a sequence of decisions, we introduce the symbol $\Phi_t \coloneqq (\phi(\xs_1),\dots,\phi(\xs_{t-1}))^\top$. We define the \emph{Gram (or kernel) matrix} $\Ks_t : \Reals^{t-1}  \rightarrow \mathbb{R}^{t-1}$ as $\Ks_t =  (k(\xs_i,\xs_j))_{i,j \in \dsb{t-1}} = \Phi_t \Phi_t^\top$. The \emph{information gain} $\Gamma_t$ and the \emph{maximal information gain} $\gamma_t$ are defined, respectively as~\citep{srinivas2010gaussian}: $
    \Gamma_t \coloneqq \frac{1}{2} \log \det ( \mathbf{I} + \lambda^{-1} \Ks_t)$ and $\gamma_t \coloneqq \max_{\xs_1,\dots,\xs_{t-1} \in \Xs}\Gamma_t$, 
where $\lambda > 0$. It can be proved that $\Gamma_t$ is the \emph{mutual information} 
$I(\mathbf{f}_t ;\mathbf{y}_t )$
between the random vectors $\mathbf{f}_t \sim \mathcal{N}(\mathbf{0}, \nu^2\mathbf{K}_t)$ and $\mathbf{y}_t = \mathbf{f}_t + \bm{\epsilon}_t$ where $\bm{\epsilon}_t \sim \mathcal{N}(\mathbf{0}_t,v^2 \lambda \mathbf{I})$, for arbitrary $v > 0$ \citep{srinivas2010gaussian}. We use the abbreviation $\mathbf{K}_t(\lambda) \coloneqq  \Ks_t + \lambda \mathbf{I} $, so that, $ \Gamma_t \coloneqq \frac{1}{2} \log \det (\lambda^{-1} \Ks_t(\lambda))$.\footnote{Known bounds of $\gamma_t$ for commonly used kernels are available in~\citep{srinivas2010gaussian,VakiliKP21}.}

\paragraph{Covariance Operators.}
Let $\Hs$ be an RKHS with kernel $k$ inducing the feature map $\phi$, let $t \in \Nat$, and let  $\xs_1,\dots,\xs_{t-1} \in \Xs$ be a sequence of decisions, the \emph{covariance operator} $V_t : \Hs \rightarrow \Hs$ is defined as $V_t =  \sum_{s=1}^{t-1} \phi(\xs_s)\phi(\xs_s)^\top = \Phi_t^\top \Phi_t$. We also introduce the operator $
    V_t(\lambda) \coloneqq V_t + \lambda I$, where $\lambda > 0$.
If $\phi(\xs)$ has bounded norm for every $\xs \in \Xs$, then $V_t= \Phi_t^\top \Phi_t$ is a trace-class operator (actually, it has finite rank) for every $t \in \Nat$ and has the same eigenvalues of $\mathbf{K}_t=\Phi_t \Phi_t^\top$. Thus, the following identity between Fredholm and matrix determinants holds \citep{whitehouse2023sublinear}:
\begin{align}\label{eq:identity}
    \det(\lambda^{-1}V_t(\lambda)) = \det(\lambda^{-1}\mathbf{K}_t(\lambda)).
\end{align}

\paragraph{Canonical Exponential Family Models.}
Let $f : \Xs \rightarrow \Reals$, a real-valued random variable $y$ belongs to the \emph{canonical exponential family}~\citep[EF,][]{barndorff2014information} if it has density:
\begin{align}\label{eq:density}
    p(y|\xs;f) = \exp \left( \frac{y f(\xs) - m(f(\xs))}{\gtau} + h(y,\tau) \right), \qquad \forall \xs \in \Xs, \; \forall y \in \Reals,
\end{align}
where $\tau > 0$ is a temperature parameter and $\mathfrak{g},m : \Reals \rightarrow \Reals$ and $h : \Reals^2\rightarrow \Reals$ are suitably defined functions~\citep{lee2024unified}. This EF model allows representing a variety of distributions, including Gaussian, Bernoulli, exponentials, and Poisson. Function $m$ is called \emph{log-partition function} that, as customary~\citep{lee2024unified,sawarnigeneralized}, is assumed to be three times differentiable and convex.
We define the \emph{inverse link function} $\mu = m'$, that, since $m$ is convex, is monotonically non-decreasing. Thus, the following hold~\citep{lee2024unified}: $\E[y|\xs;f] = m'(f(\xs)) = \mu(f(\xs))$ and $\Var[y|\xs;f] = \gtau^{-1}\dot{\mu}(f(\xs))$.
When $f$ is a linear function, the model in Equation~\eqref{eq:density} is also called \emph{generalized linear model}~\citep[GLM,][]{mccullagh2019generalized}. We also define the maximum slope of $\mu$, i.e., $R_{\dot{\mu}} \coloneqq \sup_{f \in \Hs, \xs \in \Xs} \dot{\mu}(f(\xs))$.

\section{Generalized Kernelized Bandits}\label{sec:prob}

In this section, we propose the \emph{generalized kernelized bandits} (GKBs), introducing the problem formulation  (Section \ref{sec:gkb}) and the quantities characterizing the setting (Section \ref{sec:char}).

\subsection{Problem Formulation}\label{sec:gkb}

\paragraph{Setting.}
Let $f^* \in \Hs$ be an unknown function belonging to an RKHS $\Hs$. At every round $t \in \dsb{T}$, being $T \in \Nat$ the learning horizon, the learner chooses a decision $\xs_t \in \Xs$ by means of a policy $\pi_t : \mathcal{I}_{t} \rightarrow \Xs$, being $\mathcal{I}_{t} = (\xs_1,y_1,\dots,\xs_{t-1},y_{t-1}) \in (\Xs \times \Reals)^{t-1}$ the
history realized so far, and observes a reward $y_t \sim p(\cdot|\xs_t;f^*)$. The goal of the agent is to find a decision $\xs^* \in \Xs$ maximizing the expected reward, i.e., 
$
    \xs^* \in \argmax_{\xs \in \Xs} \mu(f^*(\xs)).
$
Since $\mu$ is monotonically non-decreasing, maximizing $\mu(f^*(\cdot))$ is equivalent to maximizing $f^*(\cdot)$. Notice that GKBs generalize two well-known settings: ($i$) \emph{generalized linear bandits}~\citep[GLBs,][]{li2017provably} when the kernel is linear $k(\xs,\xs') = \inner{\xs}{\xs'}$  with $\xs,\xs' \in \Xs$ and ($ii$) \emph{kernelized bandits}~\citep[KBs,][]{chowdhury2017kernelized} when the inverse link function is the identity, i.e., $\mu = I$.

\paragraph{Regret.}
A learner, i.e., a sequence of policies $\pi = (\pi_t)_{t \in \dsb{T}}$, is evaluated by its \emph{cumulative regret} 
$R(\pi,T) \coloneqq \sum_{t \in \dsb{T}} \left( \mu(f^*(\xs^*)) - \mu(f^*(\xs_t))\right)$, where $\xs_t = \pi_t(\mathcal{I}_{t})$ for every  $ t \in \dsb{T}$.

\paragraph{Assumptions.}
We introduce the following assumptions on function $f^*$ and kernel $k$.

\begin{asm}[Bounded Norm]\label{ass:boundedNorm} There exists a known  constant $B <+\infty$ such that $\|f^*\| \le B$.
\end{asm}

\begin{asm}[Bounded Kernel]\label{ass:boundedKernel}
There exists a known constant $K <+\infty$ such that $\sup_{\xs \in \Xs} k(\xs,\xs) \le K^2$.
\end{asm}

Assumption \ref{ass:boundedKernel} ensures that $\|\phi(\xs)\| = \sqrt{k(\xs,\xs)} \le K$ and, consequently, the covariance operator $V_t = \Phi_t^\top \Phi_t$ is trace-class.
The two assumptions allow bounding the $L_\infty$-norm of $f^*$ as $\|f^*\|_{\infty} \le \|f\| \sup_{\xs \in \Xs}\| \phi(\xs) \| \le BK$.
Furthermore, these assumptions are widely employed in the KB literature~\citep{chowdhury2017kernelized}, where, in particular, Assumption~\ref{ass:boundedKernel} is enforced with $K = 1$ and is fulfilled by commonly used kernels (e.g., Gaussian and Matérn kernels). In the GLBs, where  $k(\xs,\xs') = \inner{\xs}{\xs'}$ with $\xs,\xs' \in \Xs$, Assumption \ref{ass:boundedNorm} corresponds to requiring the boundedness of the parameter vector since $\|f^*\| = \| \vtheta^*\|_2$ \citep[e.g.,][Assumption 1]{faury2020improved}, whereas Assumption~\ref{ass:boundedKernel} corresponds to requiring the boundedness of the decision vectors since $k(\xs,\xs) = \|\xs\|^2_2$ \citep[e.g.,][Assumption 2]{faury2020improved}.
 
Concerning the EF reward model, we make the following assumptions.
\begin{asm}[Bounded noise]\label{asm:mAsm}
Let $\xs \in \Xs$, $y \sim p(\cdot|\xs;f^*)$, and let $\epsilon \coloneqq y - \mu(f^*(\xs))$ be the noise. There exists a known constant $R < + \infty$ such that $|\epsilon|\le R$ almost surely.
\end{asm}
This assumption is widely used in the GLB literature~\citep{abeille2021instance,sawarnigeneralized}. If, instead of bounded noise, we deal with $\nu^2$-subgaussian noise, we can take $R = \nu \sqrt{2 \log (2T/\delta)}$ to ensure that $|\epsilon_t| \le R$ uniformly for $t \in \dsb{T}$ w.p. $1-\delta$.\footnote{Indeed, for subgaussian noise, we have that $\Pr(\exists t \in \dsb{T} \,:\,|\epsilon_t| > \nu \sqrt{2 \log (2T/\delta)}) \le \sum_{t \in \dsb{T}} \Pr(|\epsilon_t| > \nu \sqrt{ 2\log (2T/\delta)}) \le  \delta $. This requires the knowledge of the learning horizon $T$ and will result in an additional logarithmic term in the final regret bound only.} Finally, we introduce the \emph{generalized self-concordance} property~\citep{russac2021self}.
\begin{asm}[(Generalized) Self-concordance]\label{asm:selfConc}
There exists a known constant $R_{s} <+\infty$ such that for every function $f \in \Hs$ and decision $\xs \in \Xs$, it holds that $|\ddot{\mu}(f(\xs))| \le R_{s}\dot{\mu}(f(\xs))$.
\end{asm}
\citet[][Lemma 2.1]{sawarnigeneralized} shows that, if the EF model generates random variables that are bounded by $|y| \le Y$ a.s., then Assumption~\ref{asm:selfConc} holds with $R_s = Y$. Moreover, it holds for Bernoulli rewards with $R_s = 1$ and Gaussian with $R_s = 0$~\citep{lee2024unified}.

\subsection{Problem Characterization}\label{sec:char}
\paragraph{Non-Linearity of the Inverse Link Function.}
We define the following quantities that characterize the difficulty of the learning problem w.r.t.~the non-linearity of the inverse link function $\mu$: $\kappa_{*} = \frac{1}{\dot{\mu}(f^*(\xs^*))}$ and $\kappa_{\Xs} = \sup_{\xs \in \Xs} \frac{1}{\dot{\mu}(f^*(\xs))}$.
Clearly, it holds that $\kappa_{*} \le \kappa_{\Xs} $.

\paragraph{Weighted Kernel and Weighted Information Gain.}
We now discuss how the combination of a function $f \in \Hs$ and an inverse link function $\mu$ induces a new RKHS that can be characterized by its \emph{weighted kernel}. Let $f \in \Hs$, we define the weighted feature map (now dependent on $f$) as $\widetilde{\phi}(\xs;f) \coloneqq \sqrt{{\dot{\mu}(f(\xs))}{\gtau^{-1}}}\phi(\xs)$, for every $ \xs \in \Xs$.  Let $t \in \Nat$ and let $\xs_1,\dots,\xs_{t-1} \in \Xs$ be a sequence of decisions, we introduce the symbol $\widetilde{\Phi}_t(f) = (\widetilde{\phi}(\xs_1;f),\dots,\widetilde{\phi}(\xs_{t-1};f))^\top$.
This allows introducing the \emph{weighted covariance operator} $\Hop(f) : \Hs \rightarrow \Hs$ as the covariance operator induced by the feature map $\widetilde{\phi}(\cdot;f)$, i.e., $\Hop(f) = \sum_{s=1}^{t-1}  \widetilde{\phi}(\xs_s;f)  \widetilde{\phi}(\xs_s;f)^\top = \widetilde{\Phi}_t(f)^\top \widetilde{\Phi}_t(f)$, which is trace-class under Assumption \ref{ass:boundedKernel} and since $\dot{\mu}(f(\xs)) \le R_{\dot{\mu}}$, and its regularized version as
$\Hop(\lambda;f) = \Hop(f) + \lambda I$. Then, we introduce the \emph{weighted kernel} $\widetilde{k}(\cdot,\cdot;f): \Xs \times \Xs \rightarrow \mathbb{R}$, defined as:
\begin{align}
    \widetilde{k}(\xs,\xs';f) \coloneqq \inner{\widetilde{\phi}(\xs;f)}{\widetilde{\phi}(\xs';f)} =  \sqrt{\frac{\dot{\mu}(f(\xs))}{\gtau}} k(\xs,\xs') \sqrt{\frac{\dot{\mu}(f(\xs'))}{\gtau}}, \quad \forall \xs,\xs' \in \Xs.
\end{align}
This is, in all regards, a valid kernel, as it is obtained by starting from a valid kernel $k$ and performing a legal transformation~\citep{smola1998learning}.
This way, the \emph{weighted Gram (or kernel) matrix} $\widetilde{\mathbf{K}}_t(f): \Reals^{t-1} \rightarrow \Reals^{t-1}$ is defined as  $\widetilde{\mathbf{K}}_t(f) = (\widetilde{k}(\xs_i,\xs_j; f))_{i,j \in \dsb{t-1}} = \widetilde{\Phi}_t(f) \widetilde{\Phi}_t(f)^\top$ and its regularized version is $\widetilde{{\mathbf{K}}}_t(\lambda;f) = \widetilde{{\mathbf{K}}}_t(f) + \lambda \mathbf{I}$. Using the identity in Equation~\eqref{eq:identity}, we deduce that $\det(\lambda^{-1}\Hop(\lambda;f)) = \det(\lambda^{-1}\widetilde{\mathbf{K}}_t(\lambda; f))$. We also define the \emph{weighted information gain} $\widetilde{\Gamma}_t(f)$ and the \emph{weighted maximal information gain} $\widetilde{\gamma}_t(f)$ as $\widetilde{\Gamma}_t(f) \coloneqq \frac{1}{2} \log \det (\lambda^{-1} \widetilde{\Ks}_t(\lambda;f))$ and $\widetilde{\gamma}_t(f) \coloneqq \max_{\xs_1,\dots,\xs_{t-1} \in \Xs} \widetilde{\Gamma}_t(f)$, respectively.
Finally, we consider the maximum value of the (maximal) information gain by varying the function $f$ in the RKHS $\Hs$, i.e., $\widetilde{\Gamma}_t(\Hs) = \sup_{f \in \Hs}\widetilde{\Gamma}_t(f)$ and $\widetilde{\gamma}_t(\Hs) = \sup_{f \in \Hs}\widetilde{\gamma}_t(f)$, respectively.\footnote{A notion similar to the information gain is the \emph{effective dimension}, defined in the context of kernelized reinforcement learning \citep{yang2020reinforcement}, kernelized contextual bandits \citep{valko2013finite}, and neural bandits \citep{zhou2020neural,zhang2021neural,bae2025neural}.} The following result relates the weighted and the unweighted (maximal) information gains.

\begin{restatable}{lemm}{lemmaGammaGamma}\label{lemma:lemmaGammaGamma}
    Let $\Hs$ be an RKHS with kernel $k$, $t \in \Nat$, and $\xs_1,\dots,\xs_{t-1} \in \Xs$ be a sequence of decisions. It holds that $ \widetilde{\Gamma}_t(\Hs) \le \max\{1,R_{\dot{\mu}}\gtau^{-1}\} \Gamma_t$ and $ \widetilde{\gamma}_t(\Hs) \le \max\{1,R_{\dot{\mu}}\gtau^{-1}\} \gamma_t$.
\end{restatable}%
Notice that the bound introduces just a dependence on the maximum slope of the inverse link function $R_{\dot{\mu}}$ and no dependence on the maximum non-linearity index $\kappa_{\Xs}$. This result will play a role in the derivation of the efficient implementation for \algnameshort (Section~\ref{sec:efficient}).

In the following, we aim to devise algorithms with a regret bound whose dominant term depends on $\kappa_{*}$ and $\widetilde{\gamma}_{T}$ (or $\gamma_T$) as quantities that characterize the complexity of the instance.

\begin{algorithm}[t]
\fbox{\small\parbox{.9\linewidth}{
        \KwIn{Decision set $\Xs$,  confidence level $\delta$,
        confidence sets $\Cset$}

        \For{$t \in \dsb{T}$}{

            $\displaystyle \fhat \in \argmin_{f \in \Hs} \eL(f) \;\; \text{(Equation~\ref{eq:loss})}$  \hfill \texttt{$\triangleright$ Maximum Likelihood Estimate}

            $\displaystyle  (\ftilde,\xs_t) \in \argmax_{f \in \mathcal{C}_t(\delta),\, \xs \in \Xs} \mu(f(\xs)) \;\; \text{(Equation~\ref{eq:optChoice})}$ \hfill \texttt{$\triangleright$ Optimistic Decision Selection} 

            Play $\xs_t$ and observe $y_t$
        }}}
        \caption{\algnameshort.}\label{alg:1}
\end{algorithm}

\section{\algnameshort: Algorithm}\label{sec:alg}
In this section, we present \algname (\algnameshort, Algorithm~\ref{alg:1}), a regret minimization optimistic algorithm for GKBs, comprising two steps: a \emph{maximum likelihood} (ML) estimate and an \emph{optimistic decision selection}.

\paragraph{Maximum Likelihood Estimate.}
At each round $t \in \dsb{T},$  we employ the samples collected so far $\mathcal{I}_t \in (\Xs \times \Reals)^{t-1}$, to obtain an estimate $\widehat{f}_t \in \Hs$ of the  unknown $f^* \in \Hs$. Starting from the EF model, we minimize the \emph{Ridge-regularized log-likelihood}, defined as:
\begin{align}\label{eq:loss}
    \eL(f) \coloneqq \sum_{s=1}^{t-1} \frac{-y_s f(\xs_s) + m(f(\xs_s))}{\gtau} + \frac{\lambda}{2} \| f \|^2, \quad \forall f \in \Hs, \, t \in \dsb{T},
\end{align}
with $\lambda > 0$. The ML estimate is denoted as $\fhat \in \argmin_{f \in \Hs} \eL(f)$. Since the unknown true function $f^*$ belongs to the RKHS $\mathcal{H}$, we can express it as $f^* (\xs)= \inner{f^*}{k(\xs,\cdot)} = \inner{f^*}{\phi(\xs)}$ for every $\xs \in \Xs$. Thus, observing that the loss function $\eL(f)$ is twice Fréchet differentiable w.r.t. $f$, we introduce the function $g_t(f) \in \mathcal{H}$ related to its gradient w.r.t. $f$ and the \emph{weighted covariance operator} $\Hop(\lambda;f) : \mathcal{H} \rightarrow \mathcal{H}$ corresponding to its Hessian w.r.t. $f$, computed as:
\begin{align}
    & g_t(f) \coloneqq \sum_{s =1}^{t-1} \frac{{\mu}(f(\xs_s))}{\gtau} \phi(\xs_s) + \lambda f, \qquad \nabla \eL(f) = - \sum_{s=1}^{t-1} \frac{y_s \phi(\xs_s)}{\gtau} + g_t(f), \\
    & \Hop(\lambda;f) \coloneqq \nabla^2 \eL(f) = \Hop(f) + \lambda I = \sum_{s=1}^{t-1} \frac{\dot{\mu}(f(\xs_s))}{\gtau} \phi(\xs_s) \phi(\xs_s)^\top + \lambda I.\label{eq:hessian}
\end{align}
The loss function $\eL$, function $g_t$, and the operator $\Hop$ (bounded under Assumption \ref{ass:boundedKernel}) reduce to the ones employed for GLBs under the assumption that the kernel $k$ is the linear one~\citep{faury2020improved,abeille2021instance,lee2024unified}. Furthermore, if $\mu = I$, we have that $ \Hop(\lambda;f) = V_t(\lambda)$, i.e., the covariance operator, as in \citep{chowdhury2017kernelized}. 

\paragraph{Optimistic Decision Selection.}
Then, the algorithm chooses an optimistic function $\ftilde \in \Hs$, belonging to a suitable confidence set $\mathcal{C}_t(\delta)$, and an optimistic decision $\xs_t \in \Xs$:
\begin{align}\label{eq:optChoice}
    (\ftilde,\xs_t) \in \argmax_{f \in \mathcal{C}_t(\delta),\, \xs \in \Xs} \mu(f(\xs)).
\end{align}
Since $\mu$ is non-decreasing, we can ignore $\mu$ in the maximization. We consider a confidence set, defined, for every round $t \in \dsb{T}$ and confidence $\delta \in (0,1)$, as follows:
\begin{align}\label{eq:cset}
    \mathcal{C}_t(\delta) = \Big\{ f \in \Hs\,:\, \big\| g_t(f) - g_t(\fhat) \big\|_{\Hop^{-1}(\lambda;f)} \le B_t(\delta; f) \Big\},
\end{align}
where the confidence radius $B_t(\delta; f)$ will be specified later to guarantee that  $f^*$ belongs to $\mathcal{C}_t(\delta)$ in high probability and  to control the regret. 

It is worth noting that both the steps of \algnameshort require solving optimization problems in the RKHS space $\Hs$. In Section~\ref{sec:efficient}, we provide a tractable implementation of \algnameshort.

\addtocounter{footnote}{-1}

\newcommand{\hoeffding}{\textcolor{orange}{{Hoeffding}}\@\xspace}
\newcommand{\bernstein}{\textcolor{teal}{{Bernstein}}\@\xspace}
\newcommand{\freedman}{\textcolor{red}{{Freedman}}\@\xspace}
\newcommand{\pacbayes}{\textcolor{violet}{{PAC-Bayes}}\@\xspace}
\newcommand{\pseudomax}{\textcolor{blue}{{Method of mixtures}}\@\xspace}
\begin{table}[t!]
\small
    \centering
    \renewcommand{\arraystretch}{1}
    \setlength{\tabcolsep}{8pt}
    \resizebox{\linewidth}{!}{
\begin{tabular}{l|l|ccccc|l}
    \toprule
    & & \multicolumn{5}{c|}{\textbf{Properties}} & \\
    \cmidrule(lr){3-7}
    \textbf{Self-normalized bounds}
    & \textbf{Condition}
    & \rotatebox{90}{\parbox{3cm}{{Designed for Hilbert spaces}}} 
    & \rotatebox{90}{\parbox{3cm}{Dimension-free bound}} 
    & \rotatebox{90}{{Empirical bound}} 
    & \rotatebox{90}{\parbox{3cm}{{Variance-weighted normalization}}} 
    & \rotatebox{90}{\parbox{3cm}{{Decision-dependent variance}}} 
    & \textbf{Technique} \\
    \midrule\midrule
        \citealp{dani2008stochastic} & \hoeffding & \no& \no &  \no & - & -  & \freedman \\
        \citealp{abbasiyadkori2011improved} & \hoeffding & \yes{${}^*$} & \yes & \yes & - & -  & \pseudomax \\
        \citealp{chowdhury2017kernelized} & \hoeffding & \yes{${}^\dagger$} & \yes & \yes & - & - & \pseudomax \\
        \citealp{faury2020improved} & \bernstein & \no & \no & \yes & \yes & \yes  &\pseudomax \\
        \citealp{zhou2021nearly} & \bernstein & \no & \no & \no & \no & \no  &\freedman \\
        \citealp{ZhaoHZZG23} & \bernstein & \no & \no & \no & \no & \yes  &\freedman \\
        \citealp{lee2024unified} & \bernstein & \no & \no & \no & \yes & \yes  & \pacbayes \\
        \citealp{ziemann2024vector} & \bernstein & \no & \no{${}^\ddagger$} & \yes & \no & \yes & \pacbayes \\
        \citealp{bae2025neural}${}^{\clubsuit}$ & {\bernstein} & \no{${}^{\S}$} & \yes & \yes & \yes & \yes & {\freedman} \\
        \citealp{akhavan2025bernstein}${}^{\clubsuit}$ & {\bernstein} & \yes & \yes & \yes & \yes & \no{${}^{\P}$} & {\pseudomax} \\
        \midrule
        \textbf{Our work} & \textbf{\bernstein} & \yes & \yes & \yes & \yes & \yes & \textbf{\freedman} \\
        \bottomrule
    \end{tabular}}

    \begin{center}
    \footnotesize
    ${}^{\clubsuit}$ Work concurrent with ours.\\
    ${}^*$ Corollary 3.5 of \cite{abbasi2013online} extends to separable Hilbert spaces.  ${}^{\dagger}$ Only for RKHS. ${}^{\ddagger}$ The dependence on the dimension is present but hidden inside the log-determinant (Equation 1.4). ${}^{\S}$ Although it considers finite-dimensional vector spaces, the proof technique can be adapted to Hilbert spaces. ${}^{\P}$ The bound depends on the log determinant of the (non-weighted) covariance operator $V_t(\lambda)$, although normalization is performed with the weighted covariance operator $\widetilde{V}_t(\lambda)$.
    \end{center}
    \caption{\vspace{-.4cm}Summary of the properties of self-normalized concentration inequalities. \quotes{Condition} refers to the assumption on the moment-generating function of the reward noise $\epsilon_t = y_t - \mu(f^*(\xs_t))$ used to derive the inequality; 
    \quotes{Designed for Hilbert spaces} indicates that the bound is natively designed to deal with martingales taking values in Hilbert spaces;
    \quotes{Dimension-free bound} indicates that the bound does not exhibit an explicit dependence on the dimension $\mathrm{dim}(\Hs)$ of the Hilbert/vector space $\mathcal{H}$; 
    \quotes{Empirical bound} means that the bound depends on the actual sequence of decisions $\xs_1, \dots, \xs_{t-1}$ (rather than on a worst-case one);
    \quotes{Variance-weighted normalization} denotes that normalization is performed via a weighted covariance operator $\widetilde{V}_t(\lambda;f)$ (rather than the unweighted one $V_t(\lambda)$); 
    \quotes{Decision-dependent variance bound} means that the bound allows for different decisions to have different variances and explicitly captures this (rather than assuming that all decisions share the same variance or are uniformly bounded by a single variance limit);\protect\footnotemark~
    \quotes{Technique} refers to the main technical tool employed in deriving the inequality.}
    \label{tab:summary_concentration_properties}
\end{table}

\section{\algnameshort: Challenges and New Technical Tools}\label{sec:challenges}
In this section, we discuss the challenges for deriving sensible regret guarantees for \algnameshort. We start by discussing the limitations of existing {self-normalized} inequalities (Section~\ref{sec:oldC}) and, then, we derive our novel Bernstein-like dimension-free inequality, a key contribution of this work (Section~\ref{sec:newC}). A comparison between existing inequalities is reported in Table \ref{tab:summary_concentration_properties}.

\subsection{Limitations of Existing Self-Normalized Concentration Inequalities}\label{sec:oldC}
To understand the need for a novel inequality, we need to anticipate some passages of the regret analysis. We recall that the confidence radius $B_t(\delta;f)$ should be designed to guarantee that: ($i$) the true unknown function $f^*$ belongs to $\Cset$ (Equation~\ref{eq:cset}) and ($ii$) the regret is as small as possible. For point ($i$), we can express the difference between the function $g_t(\cdot)$ evaluated in the true unknown function $f^*$ and in the ML estimate $\fhat$ as (see Lemma~\ref{lemma:optimism}):
\begin{align}
     g_t(f^*) - g_t(\fhat) =  \gtau^{-1}{\sum_{s=1}^{t-1} \epsilon_s \phi(\xs_s)} + \lambda f^*, 
\end{align}
where $\epsilon_s = y_s - \mu(f^*(\xs_s))$ is the noise.
Thus, since $f^*$ has bounded norm under Assumption~\ref{ass:boundedNorm}, to design $B_t(\delta;f)$, we need to control the martingale $S_t = \sum_{s=1}^{t-1} \epsilon_s \phi(\xs_s)$, taking values in $\Hs$. For point ($ii$), we need to bound the difference between optimistic function $\widetilde{f}_t$ and true unknown function $f^*$, both evaluated in the played decision $\xs_t$, i.e., $\widetilde{f}_t(\xs_t) - f^*(\xs_t)$ with the martingale $S_t$. This is done by applying the reproducing property, i.e.,  $f^*(\xs) = \inner{\phi(\xs)}{f^*}$ for every $\xs \in \Xs$, and then applying a Cauchy-Schwarz inequality with a \emph{specific} choice of operator $W_t(f^*) : \Hs \rightarrow \Hs$, possibly depending on $f^*$ itself:
\begin{align}
    \widetilde{f}_t(\xs_t) - f^*(\xs_t) = \inner{\widetilde{f}_t-f^*}{\phi(\xs_t)} \le \underbrace{\big\| \widetilde{f}_t-f^* \big\|_{W_t(f^*)}}_{\text{(A)}} \underbrace{\left\| \phi(\xs_t) \right\|_{W_t^{-1}(f^*)}}_{\text{(B)}}.
\end{align}
\footnotetext{Clearly, it only makes sense to discuss these last two properties for inequalities derived under the Bernstein condition. Moreover, normalization using the weighted covariance operator is the key to remove the dependence on the maximum non-linearity index $\kappa_{\Xs}$ of the inverse link function $\mu$.}
The choice of operator $W_t(f^*)$ has two effects: ($i$) by relating term (A) with the confidence set $\Cset$, it determines the multiplicative constant and the norm under which $S_t$ has to be controlled and ($ii$) it allows bounding (B) with an \emph{elliptic potential lemma} (see Lemma \ref{lemma:elliptic}). We now discuss two choices of operators $W_t(f^*)$ leading to different concentration bounds and, consequently, confidence sets, and discuss their advantages and disadvantages. 

\paragraph{Covariance Operator ($W_t(f^*) = V_t(\lambda)$).}
We start considering the case in which   $W_t(f^*) = V_t(\lambda)$ is the (unweighted) covariance operator.
In this case, we can link the term (A) with the confidence set as follows (see Lemma~\ref{lemma:relationAlphaG}):
\begin{align}\label{eq:bad}
    \text{(A)} = \| \widetilde{f}_t - f^* \|_{V_t(\lambda)} \le (1+2R_sBK) \max\{1,\gtau{\textcolor{red}{\kappa_{\Xs}}}\}  \big\| g_t(\widetilde{f}_t) - g_t(f^*) \big\|_{V_t^{-1}(\lambda)},
\end{align}
introducing an inconvenient multiplicative dependence on $ \max\{1,\gtau\kappa_{\Xs}\} $, i.e., on the maximum non-linearity index $\kappa_{\Xs}$ of the inverse link function $\mu$.
At this point, we have to control the martingale $S_t$ under the norm defined by the (unweighted) covariance operator $V_t^{-1}(\lambda)$ (i.e., {unweighted normalization}), as $ \big\|g_t(f^*) - g_t(\fhat)\big\|_{V_t^{-1}(\lambda)}  \le    \left\| S_t \right\|_{V_t^{-1}(\lambda)} + \frac{B}{\sqrt{\lambda}}$. The process $\left\| S_t \right\|_{V_t^{-1}(\lambda)}$ can be conveniently bounded by using a self-normalized inequality for subgaussian\footnote{We recall that since $|\epsilon_s|\le R$ a.s. for every $s \in \dsb{T}$, it is also $R^2$-subgaussian.} martingales (i.e., \emph{Hoeffding-like}), as in \citep{abbasiyadkori2011improved}:
 {\thinmuskip=2mu
\medmuskip=2mu
\thickmuskip=2mu
\begin{align}\label{eq:abbasi}
    \left\| S_t \right\|_{V_t^{-1}(\lambda)} \le R \sqrt{2 \log(\delta^{-1}) + \log \det (\lambda^{-1}V_t(\lambda))} = R \sqrt{2 \log(\delta^{-1}) + \log \det (\lambda^{-1}\mathbf{K}_t(\lambda))},
\end{align}}%
where the equality is obtained by Equation~\eqref{eq:identity}. It is worth noting that the second bound is also obtained by \citet{chowdhury2017kernelized}. The main advantage of these bounds is that they do not exhibit a dependence on the dimension $\mathrm{dim}(\mathcal{H})$ of the Hilbert space $\mathcal{H}$, which may be infinite in GKBs. Nevertheless, in this way, the dependence on the non-linearity index $\kappa_{\Xs}$ of the inverse link function (as in Equation~\ref{eq:bad}) becomes unavoidable in the regret. This suggests that we should prefer a different choice of operator $W_t(f^*)$.

\paragraph{Weighted Covariance Operator ($W_t(f^*) = \Hop(\lambda;f^*)$).}
The presence of $\kappa_{\Xs}$ depends on the covariance operator and emerges also in the GLBs when making the choice $W_t(f^*) = V_t(\lambda)$~\citep{faury2020improved,abeille2021instance}. The solution, in the GLB case, consists of choosing the weighted covariance operator $W_t(f^*) = \Hop(\lambda; f^*)$, where each tensor product $\phi(\xs_s)\phi(\xs_s)^\top$ is weighted by the variance $\gtau^{-1}{\dot{\mu}(f(\xs_s))}$ of the noise random variable $\epsilon_s$ (i.e., \emph{variance-weighted normalization}). This allows relating term (A) directly with the confidence set $\Cset$, avoiding the inconvenient dependence on $\kappa_{\Xs}$ (see Lemma~\ref{lemma:relationAlphaG} with $f'' = f$):
\begin{align}
    \text{(A)} = \| \widetilde{f}_t - f^* \|_{\Hop(\lambda; f^*)} \le (1+2R_sBK)  \big\| g_t(\widetilde{f}_t) - g_t(f^*) \big\|_{\Hop^{-1}(\lambda;f^*)}.
\end{align}
Now, to control the process $\left\| S_t \right\|_{\Hop^{-1}(\lambda;f^*)}$ and make effective use of the variances contained in the weighted covariance operator $\Hop(\lambda; f^*)$, we need to resort to a \emph{Bernstein-like} bound. A prototypical result in the GLB literature is the bound of \cite[][Theorem 1]{faury2020improved}:
\begin{align}\label{eq:faury}
    \left\| S_t \right\|_{\Hop^{-1}(\lambda; f^*)} \le \frac{\sqrt{\lambda}}{2} + \frac{2}{\sqrt{\lambda}} {\textcolor{red}{\mathrm{dim}(\mathcal{H})}} \log 2 + \frac{2}{\sqrt{\lambda}} \log (\delta^{-1}) + \frac{1}{\sqrt{\lambda}} \log \det (\lambda^{-1}\Hop(\lambda; f^*)),
\end{align}
where $\mathrm{dim}(\mathcal{H})$ is the dimension of the Hilbert space $\mathcal{H}$ that can be infinite for an RKHS, making the bound vacuous. Other works \citep[see, \eg][]{ziemann2024vector,whitehouse2024modern} propose similar Bernstein-like bounds, but the dependency on the dimension is still present.
%


\subsection{A Novel Bernstein-Like Dimension-Free Self-Normalized Inequality}\label{sec:newC}
From the above discussion, it should now be clear why we need a new concentration bound to control the self-normalized process \( \left\| S_t \right\|_{\Hop^{-1}(\lambda; f^*)} \). This bound should leverage the \emph{Bernstein condition} to account for the noise terms \(\epsilon_t\), which are characterized by different variances \(\gtau^{-1} \dot{\mu}(\xs_t)\), depending on the decision \(\xs_t\) (i.e., \emph{decision-dependent variance}), which are then used in the definition of the weighted covariance operator \(\Hop^{-1}(\lambda; f^*)\). This would allow to avoid the inconvenient multiplicative factor $\kappa_{\Xs}$. Furthermore, we look for a \emph{dimension-free} bound that does not depend on the dimension of the Hilbert (sub)space \(\mathcal{H}\) in which the self-normalizing operator \(\Hop(\lambda; f^*)\) is defined, since it may be infinite in the GKB setting. Finally, the resulting bound should be \emph{empirical}, i.e., computable from the sequence of decisions \(\xs_1, \dots, \xs_{t-1} \in \Xs \) made so far, as in Equations~\eqref{eq:abbasi} and~\eqref{eq:faury}.

We are now ready to state our new \emph{Bernstein-like dimension-free} self-normalized inequality. For the sake of generality, we express it for a subset $\mathcal{K}\subseteq \mathcal{H}$ of the Hilbert space made of function $\psi$ with norm bounded by $K$. We will then particularize the result for \algnameshort.
\begin{restatable}[Bernstein-Like Dimension-Free Self-Normalized Concentration]{thr}{selfNormalizedComplex}\label{thr:selfNormalizedComplex}
   Let $(\psi_t)_{t \ge 1}$ be a stochastic process valued in $\mathcal{K} \subseteq \mathcal{H}$, being $\mathcal{H}$ a Hilbert space, such that $\sup_{\psi \in \mathcal{K}} \|\psi\| \le K$ and predictable by the filtration $\mathcal{F}_{t-1}$ and let $(\epsilon_t)_{t \ge 1}$ be a real-valued stochastic process adapted to the $\mathcal{F}_t$ such that $\E[\epsilon_t|\mathcal{F}_{t-1}]=0$, $\Var[\epsilon_t|\mathcal{F}_{t-1}]=\sigma^2_t$, and $|\epsilon_t| \le R$ a.s. for every $t \ge 1$. Let:
    \begin{align}
        S_t \coloneqq \sum_{s=1}^{t-1} \epsilon_s \psi_s, \qquad \Hop(\lambda) \coloneqq \sum_{s=1}^{t-1} \sigma_s^2 \psi_s \psi_s^\top  + \lambda I.
    \end{align}
    Then, for every $\delta\!\in\!(0,1)$, with probability at least $1\!-\!\delta$ and simultaneously for all $t \!\ge\! 1$, it holds that:
    {\thinmuskip=2mu
\medmuskip=2mu
\thickmuskip=2mu
    \begin{align}
        \left\| S_t \right\|_{\Hop^{-1}(\lambda)} \le  \left( \sqrt{ 73 \log\det(\lambda^{-1}\Hop(\lambda)) } + \sqrt{3} \right)\sqrt{\log \frac{\pi^2({\rho_t}+1)^2}{3\delta}} +\frac{3RK}{ \sqrt{\lambda}} \log \frac{\pi^2({\rho_t}+1)^2}{3\delta},
    \end{align}}%
    where $ \rho_t = \max\left\{ 0 , \left\lceil \log \left(  \frac{8R^2K^2(t-1)^3}{\lambda}   \log \left(1 + \frac{K^2R^2}{\lambda} \right) \right)  \right\rceil\right\}$.
\end{restatable}
Although being of independent interest, we will apply Theorem \ref{thr:selfNormalizedComplex} with $\mathcal{K} = \{k(\cdot,\xs) =\phi(\xs) \,:\, \xs \in \Xs \} \subseteq \Hs$, $\psi_s = \phi(\xs_s)$, which satisfy $\sup_{\xs \in \Xs} \|\phi(\xs)\| \le K$, $\epsilon_s = y_s - \mu(f(\xs_s))$, and $\sigma^2_s = \gtau^{-1} \dot{\mu}(f(\xs_s))$ for every $s \ge 1$. Notice that no dependence on the dimension of the Hilbert space $\mathcal{H}$ (or of its subset $\mathcal{K}$) and no dependence on the minimum variance $\sigma_s^2$ (related to the maximum non-linearity index $\kappa_{\Xs}$) are present. Finally, the bound is empirical since it is computable from the sequence of functions $\psi_s = \phi(\xs_s)$ and the variances $\sigma_s^2$ only.

\paragraph{Related Inequalities.}
The most widely used approaches to derive self-normalized inequalities in the literature are: \emph{Freedman's inequality} \citep{freedman1975tail,dani2008stochastic,zhou2021nearly,ZhaoHZZG23,bae2025neural}, \emph{pseudomaximization} via the method of mixtures~\citep{delapena2004self,abbasiyadkori2011improved,faury2020improved,akhavan2025bernstein}, and \emph{PAC-Bayes}~\citep{lee2024unified,ziemann2024vector}.
In the last six months, concurrently with our work, two preprints addressing the problem of deriving dimension-free Bernstein-like concentration bounds have been made available online \citep{bae2025neural,akhavan2025bernstein}.
\citet[][Theorem 3.1]{bae2025neural} derive their result using Freedman's inequality as a key tool, as we do, although they focus on a specific choice of inverse link function $\mu$ (i.e., logistic) and on (finite-dimensional) vector spaces. However, we believe that their proof technique could be extended, with suitable adaptations, to general Hilbert spaces. Their result resembles our Theorem~\ref{thr:selfNormalizedComplex} but exhibits a worse dependence of order $O(\log t)$ compared to our $O(\log\log t)$, although their multiplicative constant ($8$) is slightly smaller than ours ($\sqrt{73}$) in the dominant term.
\citet[][Theorem 3.2]{akhavan2025bernstein} focus natively on separable Hilbert spaces and obtain their result using the pseudomaximization technique. However, their bound depends on the information gain $\Gamma_t = \frac{1}{2} \log \det (\lambda^{-1}V_t(\lambda))$ induced by the covariance operator, instead of on $\widetilde{\Gamma}_t(f) = \frac{1}{2} \log \det (\lambda^{-1}\widetilde{V}_t(\lambda;f))$ induced by the weighted covariance operator. Furthermore, since they consider noise $\epsilon_t$ bounded by $1$ in absolute value, their bound may hide additional dependencies on the range. Their main advantage is that the bound has a mild dependence on $t$ of the same order $O(\log\log t)$ as ours, enjoys a very tight multiplicative constant ($\sqrt{6}$) in the dominating term, and exhibits no dependence on the norm $K$ of the functions in the Hilbert space.

\section{\algnameshort: Regret Analysis}\label{sec:regret}
In this section, we provide the regret analysis of \algnameshort (Algorithm~\ref{alg:1}). We start showing that $f^*$ belongs to the confidence set $\Cset$ with high probability with a proper choice of the confidence radius $B_t(\delta;f)$ (Lemma~\ref{lemma:optimism}). Then, we move to the regret analysis (Theorem~\ref{thm:regret}).

\begin{restatable}[Good Event]{lemm}{goodEvent}\label{lemma:optimism}
Let $t \in \Nat$, $f \in \Hs$, $\lambda > 0$, and $\delta \in (0,1)$, let us define the confidence radius as:
{\thinmuskip=2mu
\medmuskip=2mu
\thickmuskip=2mu
\begin{equation}\label{eq:bValue}
   B_t(\delta;f) \coloneqq \sqrt{\lambda} B + \frac{1}{\gtau}\left(\sqrt{146 \widetilde{\Gamma}_t(f) } + \sqrt{3}\right) \! \sqrt{\log \frac{\pi^2({\rho_t}+1)^2}{3\delta}} +\frac{3RK}{ \gtau \sqrt{\lambda}} \log \frac{\pi^2({\rho_t}+1)^2}{3\delta},
\end{equation}}%
   where $ \rho_t = \max\left\{ 0 , \left\lceil \log \left(  \frac{8R^2K^2(t-1)^3}{\lambda}   \log \left(1 + \frac{K^2R^2}{\lambda} \right) \right)  \right\rceil\right\}$. Let $\mathcal{E}_\delta \coloneqq \{\forall t \ge 1 \,:\, f^* \in \mathcal{C}_t(\delta)\}$. Under Assumptions~\ref{ass:boundedNorm},~\ref{ass:boundedKernel}, and~\ref{asm:mAsm}, it holds that $\Pr(\mathcal{E}_\delta) \ge 1-\delta$.
\end{restatable}
Lemma~\ref{lemma:optimism} resorts to our self-normalized concentration inequality (Theorem~\ref{thr:selfNormalizedComplex}), together with Assumptions~\ref{ass:boundedNorm} and \ref{ass:boundedKernel}, to provide a form to the confidence radius $B_t(\delta;f)$. In particular, by exploiting the identity in Equation~\eqref{eq:identity}, we replace $ \frac{1}{2}\log\det(\lambda^{-1}\Hop(\lambda; f))$ with $\widetilde{\Gamma}_t(f)  =  \frac{1}{2} \log\det(\lambda^{-1}\widetilde{\mathbf{K}}_t(\lambda;f))$. It is worth noting that, differently from the majority of existing works~\citep{abbasiyadkori2011improved,abeille2021instance,lee2024unified}, $B_t(\delta;f)$ explicitly depends on function $f$ since the operator $\Hop(\lambda; f)$ necessitates $f$ to compute the weights $\gtau^{-1} \dot{\mu}(f(\xs_s))$.  Let us also define the worst-case version of the confidence radius w.r.t. the choice of function $f \in \Hs$, i.e., $B_t(\delta;\Hs) = \sup_{f \in \Hs} B_t(\delta;f)$, obtained by replacing $\widetilde{\Gamma}_t(f)$ with $\widetilde{\Gamma}_t(\Hs)$. Although \algnameshort makes use of the confidence radius $B_t(\delta;f)$, for analysis purposes, we also define a maximal confidence radius, where the information gain $\widetilde{\Gamma}_t(f)$ is replaced by its maximal version $\widetilde{\gamma}_t(f)$:
{\thinmuskip=1mu
\medmuskip=1mu
\thickmuskip=1mu
\begin{align*}
   \beta_t(\delta;f) &\coloneqq \max_{\xs_1,\dots,\xs_{t-1} \in \Xs} B(\delta;f) \\ &= \sqrt{\lambda} B + \frac{1}{\gtau}\left(\sqrt{146 \widetilde{\gamma}_t(f) } + \sqrt{3}\right)\sqrt{\log \frac{\pi^2({\rho_t}+1)^2}{3\delta}} +\frac{3RK}{\gtau \sqrt{\lambda}} \log \frac{\pi^2({\rho_t}+1)^2}{3\delta},
\end{align*}}%
and, finally, we introduce its worst-case version w.r.t. the choice of function $f \in \Hs$, i.e., $\beta_t(\delta;\Hs) = \sup_{f \in \Hs} \beta_t(\delta;f)$, obtained from $ \beta_t(\delta;f)$ by replacing $\widetilde{\gamma}_t(f)$ with $\widetilde{\gamma}_t(\Hs)$.

We are now ready to present the regret bound of \algnameshort.

\begin{restatable}[Regret Bound of \algnameshort]{thr}{regretBound}\label{thm:regret}
Under Assumptions~\ref{ass:boundedNorm}, \ref{ass:boundedKernel}, \ref{asm:mAsm}, and \ref{asm:selfConc}, \algnameshort with $B_t(\delta;f)$ defined in Equation \eqref{eq:bValue} and $\lambda > 0$, for every $\delta \in (0,1)$, with probability at least $1-\delta$, suffers regret $
    R(\text{\algnameshort},T) = R_{\textnormal{\text{perm}}}(T) + R_{\textnormal{\text{trans}}}(T)$, bounded as:
\begin{align}
    & R_{\textnormal{\text{perm}}}(T) \le 8  (1+2R_sB{K}){\beta}_T(\delta;\Hs) \sqrt{\max \left\{\gtau, \lambda^{-1} R_{\dot{\mu}} K^2\right\}\widetilde{\gamma}_T(f^*)} \sqrt{\frac{T}{\kappa_{*}}} , \\
    & R_{\textnormal{\text{trans}}}(T) \le  32  R_s (1+ R_{\dot{\mu}} \kappa_{\Xs})  (1+2R_sBK)^2{\beta}_T(\delta;\Hs)^2 \max \left\{\gtau, \lambda^{-1} R_{\dot{\mu}} K^2\right\} \widetilde{\gamma}_T(f^*).
\end{align}
\end{restatable}%
The proof schema of Theorem~\ref{thm:regret} follows similar steps to \citep{abeille2021instance} and the result, indeed, displays an analogous regret decomposition into a \emph{permanent} term $R_{\textnormal{\text{perm}}}(T)$ and a \emph{transient} term $R_{\textnormal{\text{trans}}}(T)$. Regarding the explicit dependence on $T$ and $\kappa_{*}$, $R_{\textnormal{\text{perm}}}(T)$ dominates and displays the desired dependence on $\sqrt{{T}/{\kappa_{*}}}$. $R_{\textnormal{\text{trans}}}(T)$ exhibits a dependence on the maximum non-linearity index $\kappa_{\Xs}$ and on the information gains product $\widetilde{\gamma}_T(\Hs)\widetilde{\gamma}_T(f^*)$, but has explicit dependence on $T$ which is only logarithmic and, thus, it can be considered negligible compared to $R_{\textnormal{\text{perm}}}(T)$. To highlight the dependence on the information gain(s) in $R_{\textnormal{\text{perm}}}(T)$, we explicit the form of the individual terms in the case $\lambda \ge \Omega(K^2)$, suppressing dependencies on $\gtau$ and $R_{\dot{\mu}}$, obtaining: ${\beta}_T(\delta;\Hs) = \widetilde{O}(\sqrt{\lambda} B + \sqrt{\widetilde{\gamma}_T(\Hs)\log(\delta^{-1})} + RK\log(\delta^{-1}) )$. Thus, we obtain a regret bound of order:
{\thinmuskip=2mu
\medmuskip=2mu
\thickmuskip=2mu
\begin{equation*}
    R(\text{\algnameshort},T) \le \widetilde{O} \left( \!(1+R_sBK) \! \left(\!\sqrt{\lambda} B + \sqrt{\widetilde{\gamma}_T(\Hs)\log(\delta^{-1})} + {RK} \log(\delta^{-1})\right) \! \sqrt{\widetilde{\gamma}_T(f^*)} \sqrt{\frac{T}{\kappa_{*}}} \right).
\end{equation*}}%
Two weighted maximal information gains, $\widetilde{\gamma}_T(\Hs)$ and $\widetilde{\gamma}_T(f^*)$ are present since the weighted kernel $\widetilde{k}(\cdot,\cdot;f)$ explicitly depends on the evaluated function $f$. Thanks to Lemma~\ref{lemma:lemmaGammaGamma}, we can bound both with the (unweighted) information gain: $\widetilde{\gamma}_T(f^*) \!\le \widetilde{\gamma}_T(\Hs) \! \le \max\{1,R_{\dot{\mu}}\gtau^{-1}\} \gamma_T$, at the mild price of a multiplicative term, leading to a simplified final rate of $ R(\text{\algnameshort},T) \le \widetilde{O}(\gamma_T \sqrt{T/\kappa_*})$ when ignoring dependencies on $\lambda$, $B$, $K$, $R_s$, and $\delta$ too.

\paragraph{Discussion for KBs and GLBs.}
In the case of KBs, we have to consider $\nu^2$-subgaussian noise and, thus, we set $R= O(\nu\sqrt{\log (T/\delta)})$. Furthermore, we have that $R_s=0$ and $\mu=I$ (consequently, $\dot{\mu} = 1$, $\kappa_*=1$, and $\widetilde{\gamma}_T(f^*) = \widetilde{\gamma}_T(\Hs) = \gamma_T$). This allows recovering the bound of order  $\widetilde{O}\left(\left(\sqrt{\lambda} B + \sqrt{{\gamma}_T\log(\delta^{-1})} + K \nu \log(\delta^{-1})^{3/2} \right)\sqrt{{\gamma}_T T}\right)$, matching the regret order of \citep{chowdhury2017kernelized} up to logarithmic terms. For GLBs, we can bound the information gain as follows~\citep[see Lemma 11 of][]{abeille2021instance}:
\begin{align}
    \widetilde{\gamma}_T(\Hs)\le  \max\{1,R_{\dot{\mu}}\gtau^{-1}\} \gamma_T \le  \max\{1,R_{\dot{\mu}}\gtau^{-1}\} d \log \left(\lambda + \frac{TK^2}{d}\right).
\end{align}
This leads to bound of order $\widetilde{O}((1+R_sBK)(\sqrt{\lambda} B + \sqrt{d\log(\delta^{-1})} + RK\log(\delta^{-1}) ) \sqrt{dT/{\kappa_{*}}})$, matching the result of \citep{faury2020improved,abeille2021instance}, up to logarithmic terms. However, for GLBs, it is known that the dependence on $(1+R_sBK)$ can be removed with techniques other than self-normalized bounds \citep{lee2024unified}. We leave the adaptation of those techniques to the GKB setting, which we believe is non-trivial, to future works.

\section{\algnameeffshort: Tractable Implementation}\label{sec:efficient}
\algnameshort (Algorithm~\ref{alg:1}) uses a confidence set $\Cset$ (Equation~\ref{eq:cset}) that requires evaluating the norm of a difference of functions $g_t(\cdot) \in \Hs$, which is computationally intractable. In this section, we show how to make both steps of \algnameshort computationally tractable, thereby obtaining its efficient variant \algnameeff (\algnameeffshort, Algorithm~\ref{alg:2}). We discuss its time and space complexity and provide the regret analysis, showing that \algnameeffshort enjoys regret guarantees similar to \algnameshort.

\noindent\begin{algorithm}[t]
\noindent\fbox{\small\parbox{.9\linewidth}{\KwIn{Decision set $\Xs$, confidence level $\delta$, confidence sets $\mathcal{D}_t(\delta)$}

        \For{$t \in \dsb{T}$}{

            $\displaystyle \widehat{\bm{\alpha}}_t = \argmin_{\bm{\alpha} \in \Reals^{t-1}} \overline{\mathcal{L}}_t( \bm{\alpha}) \; \text{(Equation~\ref{eq:loss_eff})}$ \hfill \texttt{$\triangleright$ Tractable Maximum Likelihood Estimate}

            $\displaystyle  \xs_t \in \argmax_{\xs \in \Xs} \max_{\bm{\alpha} \in \Reals^{t-1}} {\bm{\alpha}}^\top{\mathbf{k}_{t}({\xs})} \;\; \hfill \texttt{$\triangleright$ Tractable Optimistic Decision Selection} \\ 
            \phantom{AAA} \text{s.t.} \;\; \overline{\mathcal{L}}_t({\bm{\alpha}}) \le \overline{\mathcal{L}}_t({\widehat{\bm{\alpha}}_t})+  D_t(\delta) \quad \text{(Equation~\ref{eq:optEff})}$

            Play $\xs_t$ and observe $y_t$
        }}}
        \caption{\algnameeffshort.}\label{alg:2}
\end{algorithm}

\vspace{-.5cm}

\paragraph{Tractable Maximum Likelihood Estimation.}
Since function $m$ is convex, loss function $\eL(f)$ is convex in $f$ as well. However, optimizing on $f$ ranging in the RKHS $ \mathcal{H}$ is intractable. Nevertheless, thanks to the \emph{generalized representer theorem}~\citep[][Theorem 1]{scholkopf2001generalized}, at every round $t \in \dsb{T}$, we can restrict the optimization to the functions of the form $f(\cdot) = \sum_{s=1}^{t-1} \alpha_s k(\cdot,\xs_s) = \bm{\alpha}^\top{\mathbf{k}_t(\cdot)}$,
where $\bm{\alpha} = (\alpha_1,\dots,\alpha_{t-1})^\top$ and $\mathbf{k}_t(\cdot) = (k(\cdot,\xs_1),\dots,k(\cdot,\xs_{t-1}))^\top$. This allows framing the problem to the minimization of a convex function $\overline{\mathcal{L}}_t$ on a vector of $t-1$ real optimization variables $\bm{\alpha} \in \Reals^{t-1}$:
\begin{align}\label{eq:loss_eff}
    \overline{\mathcal{L}}_t(\bm{\alpha}) \coloneqq  \mathcal{L}_t(\bm{\alpha}^\top{\mathbf{k}_t(\cdot)}) = \sum_{s=1}^{t-1} \frac{-y_s \bm{\alpha}^\top{\mathbf{k}_t(\xs_s)} + m(\bm{\alpha}^\top{\mathbf{k}_t(\xs_s)})}{\gtau} + \frac{\lambda}{2} \bm{\alpha}^\top \mathbf{K}_t \bm{\alpha}.
\end{align}
Having fixed $\bm{\alpha} \in \Reals^{t-1}$, if the evaluation of the kernel function $k$ takes $O(1)$, the computation of the loss function $ \overline{\mathcal{L}}_t(\bm{\alpha})$ can be perfomed with $O(t^2)$ operations and with $O(t^2)$ memory to store matrix $\mathbf{K}_t$. Moreover, $ \overline{\mathcal{L}}_t$ is convex and smooth in the optimization variables $\bm{\alpha}$ with smoothness constant $\|\nabla^2 \overline{\mathcal{L}}_t(\bm{\alpha})\|_2 \le L_t \coloneqq K (t-1) \left(\gtau^{-1}R_{\dot{\mu}}(t-1) + \lambda\right) $. Thus, we may use standard convex optimizers (e.g., gradient descent), which converge at rate $O(1/\iota)$ for convex smooth functions, where $\iota$ is the number of updates \citep{bubeck2015convex}.

\paragraph{Tractable Optimistic Decision Selection.} To perform the choice of the optimistic decision, we propose a different (looser) confidence set based on the evaluation of the loss function only~\citep{abeille2021instance}, defined for every round $t \in \dsb{T}$ and confidence $\delta \in (0,1)$:
\begin{align}\label{eq:bonus2}
    \mathcal{D}_t(\delta) \coloneqq \left\{ f \in \Hs \,:\, \eL(f)-\eL(\fhat) \le D_t(\delta) \coloneqq (1+2R_sBK) B_t'(\delta) \right\}, \quad \ \text{where:}
\end{align}
{\thinmuskip=1mu
\medmuskip=1mu
\thickmuskip=1mu
\begin{align*}
    B_t'(\delta) \coloneqq \sqrt{\lambda} B + \frac{\sqrt{146 \max\{1,R_{\dot{\mu}}\gtau^{-1}\}{\Gamma}_t } + \sqrt{3}}{\gtau} \! \sqrt{\log \frac{\pi^2({\rho_t}+1)^2}{3\delta}} +\frac{3RK}{ \gtau \sqrt{\lambda}} \log \frac{\pi^2({\rho_t}+1)^2}{3\delta},
\end{align*}}%
Notice that $B_t'(\delta)$ is obtained from the original confidence radius $B_t'(\delta;f)$ by simply upper bounding $\widetilde{\Gamma}_t(f)$ thanks to Lemma \ref{lemma:boundIG} and, thus, removing any dependence on function $f$.\footnote{Retaining the dependence on $f$ in the confidence radius would bring a minimal advantage in the final regret bound with a significant computational burden in the solution of the optimistic decision selection.} We also introduce its maximal version, i.e., $\beta_t'(\delta) \coloneqq \max_{\xs_1,\dots,\xs_{t-1} \in \Xs} B_t'(\delta) $ which is obtained from $B_t'(\delta)$ by replacing the information gain ${\Gamma}_t$ with its maximal version $\gamma_T$. Clearly, it holds that $B_t'(\delta) \ge B_t(\delta;\Hs)$ and  $\beta_t'(\delta) \ge \beta_t(\delta;\Hs)$.
We prove in Lemma~\ref{lemma:cSetSecond} that this choice of the confidence radius ensures the inclusion property $\Cset \subseteq \mathcal{D}_t(\delta)$. Having fixed a decision $\widehat{\xs} \in \Xs$,\footnote{As customary in this literature~\citep{srinivas2010gaussian,chowdhury2017kernelized}, we do not address the issue of optimizing over the decision space $\Xs$ efficiently. This can surely be done efficiently when $\Xs$ is finite. When $\Xs$ is continuous, we can resort to a \emph{discretization} based on the regularity properties of the kernel function, with a controllable effect on the final regret performances~\citep{rando2022ada}.} the optimistic decision selection can be formulated, thanks to the \emph{generalized representer theorem}~\citep{scholkopf2001generalized} as the following constrained convex program:\footnote{Even if formulated for unconstrained minimization, the representer theorem admits cost functions that take $+\infty$ as value~\citep{scholkopf2001generalized}. Thus, we can convert constrained minimization to unconstrained by bringing the constraint into the objective function and setting it to $+\infty$ if the constraint is violated.}
\begin{equation}\label{eq:optEff}
\begin{aligned}
     \min_{\bm{\alpha} \in \Reals^{t-1}} \; -\bm{\alpha}^\top{\mathbf{k}_{t}(\widehat{\xs})}  \qquad \text{subject to} \qquad \overline{\mathcal{L}}_t(\bm{\alpha}) \le  \overline{\mathcal{L}}_t(\widehat{\bm{\alpha}}_t)+ D_t(\delta),
\end{aligned}
\end{equation}
where $\widehat{\bm{\alpha}}_t \in \argmin_{\bm{\alpha} \in \Reals^{t-1}} \overline{\mathcal{L}}_t(\bm{\alpha}) $ are the ML parameters computed in the previous step.
Having fixed $\widehat{\xs} \in \Xs$ and $\bm{\alpha} \in \Reals^{t-1}$, if the evaluation of the kernel function $k$ takes $O(1)$, the computation of the objective and of the constraint take $O(t)$ and $O(t^2)$ operations, respectively, with $O(t^2)$ memory to store $\mathbf{K}_t$. The program has $t-1$ real variables, a linear objective function and one convex constraint, being $\overline{\mathcal{L}}_t$ convex in $\bm{\alpha}$ with the same smoothness constant $L_t$ computed before. Thus, similarly as above, we may use standard constrained convex optimizers (e.g., projected gradient descent), which converge at rate $O(1/\iota)$, where $\iota$ is the number of updates \citep{bubeck2015convex}. Furthermore, the information gain $\Gamma_t$, needed to compute $D_t(\delta)$, can be updated incrementally as $\Gamma_{t+1} = \Gamma_{t} + \log( 1 + \lambda^{-1} (k(\xs_{t}, \xs_{t}) - \mathbf{k}_t(\xs_t)^\top \mathbf{K}_{t}^{-1}(\lambda) \mathbf{k}_t(\xs_{t})))$  and the inverse matrix $ \mathbf{K}_{t}^{-1}(\lambda)$ can be updated incrementally as well with an overall cost of $O(t^2)$ operations \citep{chowdhury2017kernelized}.

We now show that the choice of the new confidence set $\mathcal{D}_t(\delta)$ does not degrade the regret in the dependence on the relevant quantities compared to using $\Cset$.
\begin{restatable}[Regret Bound of \algnameeffshort]{thr}{regretBoundSecond}\label{thm:regret2}
Under Assumptions~\ref{ass:boundedNorm}, \ref{ass:boundedKernel}, \ref{asm:mAsm}, and \ref{asm:selfConc}, \algnameshort with $D_t(\delta)$ defined in Equation \eqref{eq:bonus2} and $\lambda > 0$, for every $\delta \in (0,1)$, with probability at least $1-\delta$, suffers regret: $
    R(\text{\algnameeffshort},T) = R_{\textnormal{\text{perm}}}(T) + R_{\textnormal{\text{trans}}}(T)$, bounded as:
\begin{align*}
    \resizebox{.97\textwidth}{!}{$\displaystyle 
    \begin{aligned}
    &R_{\textnormal{\text{perm}}}(T) \! \le \! 4 \sqrt{\max \left\{\gtau, \lambda^{-1} R_{\dot{\mu}}K^2\right\}} (2+2R_sBK)  \sqrt{{\beta}_T'(\delta)} \! \Big( \!  \sqrt{{\beta}_T'(\delta)} \! + \!  2 \! \Big) \!  \sqrt{\widetilde{\gamma}_T(f^*)} \sqrt{ \frac{T}{\kappa_{*}}},\\
    &R_{\textnormal{\text{trans}}}(T) \! \le \! 8  R_s (1\! +\! R_{\dot{\mu}} \kappa_{\Xs}) \max \left\{\gtau, \lambda^{-1} \! R_{\dot{\mu}}K^2\right\} \! (2\!+\!2R_sBK)^2 {{\beta}_T'(\delta)\! }\Big(\!\sqrt{\!{\beta}_T'(\delta)} \! + \! 2 \Big)^2 \!\! \widetilde{\gamma}_T(f^*) .     
    \end{aligned}$}
\end{align*}
\end{restatable} 

The bounds in Theorems~\ref{thm:regret} and~\ref{thm:regret2} have the same order dependence on the relevant quantities, but Theorem~\ref{thm:regret2} replaces $\beta_T(\delta;\Hs)$ with the larger $\beta_T'(\delta)$ and has a multiplicative constant for $R_{\textnormal{\text{perm}}}(T)$ (resp. $R_{\textnormal{\text{trans}}}(T)$) that is roughly 3 (resp. 9) times larger.

\section{Conclusions} 
In this paper, we have introduced the novel setting of GKBs, unifying KBs and GLBs. We have provided a novel Bernstein-like dimension-free self-normalized concentration inequality of independent interest. We employed it to analyze the regret of \algnameshort showing tight regret bounds.  Finally, we proposed an efficient version, \algnameeffshort, which preserves the regret guarantees. 
Future works include investigating the use of the techniques from \cite{lee2024unified} to remove the multiplicative dependence on the norm and kernel bounds $(1+R_sBK)$ in the regret bound and the study of the inherent complexity of regret minimization in the GLB setting by conceiving regret lower bounds~\citep{scarlett2017lower}.


\appendix 


\section{A Data-Driven Freedman's Inequality}\label{apx:proofs}
In this appendix, we derive a data-driven Freedman's inequality and compare it with the related literature. We start by revising the standard version of Freedman's inequality.

\begin{lemm}[Freedman's Inequality]\label{lemma:freedman}
    Let $(z_t)_{t \ge 1}$ be a real-valued martingale difference sequence adapted to the filtration $\mathcal{F}_t$ such that $z_t \le R$ a.s.~for all $t \ge 1$. Then, for every $\lambda \in (0,3/R)$ it holds that with probability at least $1-\delta$:
    \begin{align}
        \forall t \ge 1: \qquad \sum_{s=1}^t z_s \le  \frac{\lambda}{2(1-\lambda R/3)} \sum_{s=1}^t \E[z_s^2|\mathcal{F}_{s-1}] + \frac{\log (\delta^{-1})}{\lambda}.
    \end{align}
    This implies that for every $\nu > 0$, with probability at least $1-\delta$:
    \begin{align}
     \forall t \ge 1: \qquad \sum_{s=1}^t z_s \le  \nu \sqrt{2 \log (\delta^{-1})} + \frac{R \log (\delta^{-1})}{3} \quad \text{or } \quad \sum_{s=1}^t \E[z_s^2|\mathcal{F}_{s-1}] > \nu^2.
    \end{align}
\end{lemm}
    
\begin{proof}
    Refer to \citep[][Theorem~13.6]{zhang2023mathematical}.
\end{proof}
We are now ready to state our result.
\begin{restatable}[A data-driven Freedman's inequality]{thr}{newFreedman}\label{thm:newFreedman}
    Let $(z_t)_{t \ge 1}$ be a real-valued martingale difference sequence adapted to the filtration $\mathcal{F}_t$ such that $z_t \le R$ a.s. for all $t \ge 1$. Let $(v_t)_{t \ge 1}$ be a real-valued process predictable by the filtration $\mathcal{F}_t$ such that for every $t \ge 1$, we have that $\sum_{s=1}^t \E[z_s^2|\mathcal{F}_{s-1}] \le v_t$ a.s.. Then, for every $\eta > 1$ and $v_0>0$, with probability at least $1-\delta$, it holds that:
    \begin{align}
         \forall t \ge 1: \qquad \sum_{s=1}^t z_s \le  \sqrt{2 \max\left\{v_0, \eta v_t \right\}\log \frac{\pi^2(\widehat{\ell}_t+1)^2}{6\delta}} + \frac{R}{3} \log \frac{\pi^2(\widehat{\ell}_t+1)^2}{6\delta},
    \end{align}
    where $\widehat{\ell}_t = \max\left\{0, \left\lceil \log_{\eta} ({v_t}/{v_0}) \right\rceil \right\}$.
\end{restatable}

    {\thinmuskip=2mu
\medmuskip=2mu
\thickmuskip=2mu
\begin{proof}
    The proof makes use of classical Freedman's inequality \citep{freedman1975tail} combined with a \emph{stitching} argument~\citep{howard2021time}.
    We start from the version of Freedman's inequality of Lemma~\ref{lemma:freedman} taken from \citep{zhang2023mathematical}:
    \begin{align}
        \Pr \left( \exists t \ge 1:\, \sum_{s=1}^t z_s >  \nu \sqrt{2 \log (\delta^{-1})} + \frac{R \log (\delta^{-1})}{3} ,\; \sum_{s=1}^t \E[z_s^2|\mathcal{F}_{s-1}] \le \nu^2 \right) \le \delta.
    \end{align}
    Since $v_t \ge \sum_{s=1}^t \E[z_s^2|\mathcal{F}_{s-1}]$ a.s. for every $t \ge 1$, it immediately follows that:
    \begin{align}\label{eq:ineqInter}
    \Pr & \left( \exists t \ge 1:\, \sum_{s=1}^t z_s >  \nu \sqrt{2 \log (\delta^{-1})} + \frac{R \log (\delta^{-1})}{3} ,\;  v_t \le \nu^2 \right) \le \delta.
    \end{align}
    We now proceed by performing a stitching argument with a geometric grid over the values of $\nu \ge 0$ defined as $\{\eta^\ell v_0 \,:\, \ell \in \mathbb{N}\}$ for any choice of $\eta >1$ and $v_0>0$. Thus, we have:
    \begin{align}
        \Pr & \left( \exists \ell \in \mathbb{N},\; \exists t \ge 1 \,:\, \sum_{s=1}^t z_s >   \sqrt{2 \eta^\ell v_0  \log \frac{\pi^2(\ell+1)^2}{6\delta}} + \frac{R}{3}  \log \frac{\pi^2(\ell+1)^2}{6\delta} ,\,  v_t \le \eta^\ell v_0 \right) \\
        & \le  \sum_{\ell \in \mathbb{N}}  \Pr \left( \exists t \ge 1 \,:\, \sum_{s=1}^t z_s >   \sqrt{2 \eta^\ell v_0  \log \frac{\pi^2(\ell+1)^2}{6\delta}} + \frac{R}{3}  \log \frac{\pi^2(\ell+1)^2}{6\delta} ,\,  v_t \le \eta^\ell v_0 \right)\label{line:0001} \\
        & \le \sum_{\ell \in \mathbb{N}}  \frac{6\delta}{\pi^2(\ell+1)^2} \le \delta,\label{line:0002}
    \end{align}
    where line~\eqref{line:0001} follows from a union bound, line~\eqref{line:0002} is an application of Equation~\eqref{eq:ineqInter} with $\nu =\eta^{\ell}v_0$ and by observing that $\sum_{\ell \in \mathbb{N}}   \frac{1}{(\ell+1)^2}  = \frac{\pi^2}{6}$. Let $\widehat{\ell}_t \in \Nat$ be the smallest value such that $v_t \le \eta^{\ell} v_0$. We have $ \widehat{\ell}_t = \min\left\{\ell \in\Nat\,:\, v_t \le \eta^{\ell} v_0\right\} = \max\left\{0, \left\lceil \log_{\eta} \frac{v_t}{v_0} \right\rceil \right\}$, for which it holds that:
    \begin{align}\label{eq:etaEll}
        \eta^{\widehat{\ell}_t} v_0 \le \eta^{\max\left\{0, \left\lceil \log_{\eta} \frac{v_t}{v_0}\right\rceil \right\}} v_0  \le \eta^{\max\left\{0,  \log_{\eta} \frac{v_t}{v_0} + 1 \right\}} v_0  \le \max\left\{v_0, \eta v_t\right\}.
    \end{align}
    Finally, we prove the inequality:
    \begin{align}
    \Pr & \left( \exists t \ge 1 \,:\, \sum_{s=1}^t z_s >  \sqrt{2 \max\left\{v_0, \eta v_t \right\}\log \frac{\pi^2(\widehat{\ell}_t+1)^2}{6\delta}} + \frac{R}{3} \log \frac{\pi^2(\widehat{\ell}_t+1)^2}{6\delta} \right) \\
    & \le  \Pr  \left( \exists t \ge 1 \,:\, \sum_{s=1}^t z_s >   \sqrt{2 \eta^{\widehat{\ell}_t} v_0  \log \frac{\pi^2(\widehat{\ell}_t+1)^2}{6\delta}} + \frac{R }{3}  \log \frac{\pi^2(\widehat{\ell}_t+1)^2}{6\delta}\;, v_t \le \eta^{\widehat{\ell}_t} v_0 \right) \label{line:0003} \\
    & \le  \Pr \left( \exists \ell \in \mathbb{N} , \; \exists t \ge 1 \,:\, \sum_{s=1}^t z_s >   \sqrt{2 \eta^\ell v_0  \log \frac{\pi^2(\ell+1)^2}{6\delta}} + \frac{R}{3}  \log \frac{\pi^2(\ell+1)^2}{6\delta} \;,  v_t \le \eta^\ell v_0 \right) \label{line:0004} \\
    &\le \delta, \nonumber
    \end{align}
    where line~\eqref{line:0003} follows from Equation~\eqref{eq:etaEll} and line~\eqref{line:0004} from line~\eqref{line:0002}.
\end{proof}}

Theorem~\ref{thm:newFreedman}, compared to the standard Freedman's inequality (see Lemma~\ref{lemma:freedman}), allows obtaining a bound that depends on the predictable process $v_t$ that we can think to as a proxy (upper bound) of the variance that, however, does not need to be deterministic. From a technical perspective, Theorem~\ref{thm:newFreedman} is obtained using a \emph{stitching} argument~\citep{howard2021time} that brings two beneficial effects. First, it allows to accurately perform \emph{union bounds} considering the values that the predictable process can take over a geometric grid $\{\eta^\ell v_0 \,:\, \ell \in\Nat\}$ enabling the use of the data-driven quantity $v_t$, where the parameters $\eta >1$ and $v_0>0$ can be selected to tighten the bound. Second, it allows replacing a $O(\log t)$ term in the bound with a $O(\log\log t)$ at the price of a larger multiplicative constant $\eta >1$. 

\paragraph{Related Inequalities.}
Several variations of Freedman's inequality have been proposed in the literature, which differ mainly in two dimensions: ($i$) the dependence on the number of random variables $t$ in the bound, and ($ii$) the knowledge of a bound $R$ on the random variables $z_s$ holding almost surely. 
\citet[][Theorem 12]{GaillardSE14} and \citet[][Lemma 3]{rakhlin2012making} assume bounded random variables and exhibit a dependence on $t$ of order $O(\log\log t)$ in the dominating term, similar to ours. 
However, our result allows tuning $\eta$ and $v_0$ to tighten the bound, leading to an improvement in the multiplicative constants. 
\citet[][Theorem 6]{li2024q}, instead, provide a bound with no explicit dependence on $t$, but where the bound depends on a parameter $K$ that appears inside the logarithm and replaces the variance-related term $V_t$ with $\max\{V_t, \sigma^2 / 2^K\}$, where $\sigma^2$ is a constant such that $V_t \le \sigma^2$ almost surely. 
In particular, by choosing $K = O(\log t)$, we recover the order $O(\log\log t)$ in the dominating term. 
Differently, \citet[][Theorem 9]{zimmert2022return} and \citet[][Theorem 2.2]{lee2020bias} provide bounds that scale with a quantity related to the empirical maximum of the involved random variables, $\max\{1, \max_{s \in \dsb{t}} z_s\}$, ultimately leading to a dependence of order $O(\log t)$. 
Notice that, because of the presence of the $\max\{1, \cdot\}$ term, when the random variables are bounded by a constant $R < 1$, the bound still scales with $1$.


\section{Proofs}

\subsection{Proofs of Section~\ref{sec:prob}}
\lemmaGammaGamma*
\begin{proof}
A direct application of Lemma~\ref{lemma:boundIG}.
\end{proof}

\subsection{Proofs of Section~\ref{sec:challenges}}

\selfNormalizedComplex*
\begin{proof}
    The proof follows similar steps as \citep{dani2008stochastic,zhou2021nearly}, using Theorem~\ref{thm:newFreedman} as base inequality, recalling that we are dealing with functions $\psi$ of a Hilbert space $\mathcal{H}$ rather than vectors. For the sake of this derivation, we will suppress the dependence on $\lambda$, simply writing $\Hop(\lambda)= \Hop$.
    Let us introduce the notation $Z_t \coloneqq \left\| S_t \right\|_{\Hop^{-1}}$, $w_t \coloneqq \|\psi_t\|_{\Hop^{-1}}$, and $\widetilde{w}_t \coloneqq \sigma_t \|\psi_t\|_{\Hop^{-1}}$. 
    
    First of all, let us recall that for every $t \in \Nat$, we have that $\widetilde{V}_t$ is a bounded self-adjoint positive definite linear operator on $\mathcal{H}$, for every $\lambda > 0$ and, thus, invertible.
    Recalling that $\Hop = \Hop[t-1] + \sigma_{t-1} \psi_{t-1} \psi_{t-1}^\top$, from the Sherman-Morris formula in Hilbert spaces~\citep{woodbury1950inverting,li2022sherman}, we have:
    \begin{align}
    \Hop^{-1} & = \Hop[t-1]^{-1} - \frac{\Hop[t-1]^{-1} \psi_{t-1}\psi_{t-1}^\top \Hop[t-1]^{-1} \sigma_{t-1}^2}{1 + \|\psi_{t-1}\|_{\Hop[t-1]^{-1}}^2 \sigma_{t-1}^2} = \Hop[t-1]^{-1} - \frac{\Hop[t-1]^{-1} \psi_{t-1}\psi_{t-1}^\top \Hop[t-1]^{-1} \sigma_{t-1}^2}{1 + \widetilde{w}_{t-1}^2}.
    \end{align}
    Let us decompose $Z_t$:
    \begin{align}
        Z_t^2 & \coloneqq \left\| S_t \right\|_{\Hop^{-1}}^2 = S_t^\top \Hop^{-1} S_t \\
        & = (S_{t-1} + \epsilon_{t-1} \psi_{t-1})^\top \Hop^{-1}(S_{t-1} + \epsilon_{t-1} \psi_{t-1}) \\
        & = S_{t-1}\Hop^{-1}S_{t-1} + 2\epsilon_{t-1} \psi_{t-1}^\top \Hop^{-1}S_{t-1} + \epsilon_{t-1}^2 \psi_{t-1}^\top   \Hop^{-1} \psi_{t-1} \\
        & \le S_{t-1}\Hop[t-1]^{-1}S_{t-1} + \underbrace{ 2\epsilon_{t-1} \psi_{t-1}^\top \Hop^{-1}S_{t-1}}_{ \text{(A)}} + \underbrace{\epsilon_{t-1}^2 \psi_{t-1}^\top   \Hop^{-1} \psi_{t-1}}_{ \text{(B)}},
    \end{align}
    having exploited the fact that $\Hop \succeq \Hop[t-1]$. We analyze terms (A) and (B) separately.
    
    \paragraph{Analysis of Term (A).}
    From the Sherman-Morris formula in Hilbert spaces, we have:
    \begin{align}
        2 \epsilon_{t-1} \psi_{t-1}^\top \Hop^{-1} S_{t-1}&= 2 \epsilon_{t-1} \Bigg( \psi_{t-1}^{\top} \Hop[t-1]^{-1}S_{t-1} - \frac{\psi_{t-1}^{\top} \Hop[t-1]^{-1} \psi_{t-1}\psi_{t-1}^\top \Hop[t-1]^{-1}S_{t-1} \sigma_{t-1}^2}{1 + \widetilde{w}_{t-1}^2} \Bigg) \\
        & = 2 \epsilon_{t-1} \left( \psi_{t-1}^{\top} \Hop[t-1]^{-1}S_{t-1} - \frac{\widetilde{w}^2_{t-1}}{1 + \widetilde{w}_{t-1}^2} \psi_{t-1}^\top \Hop[t-1]^{-1}S_{t-1}  \right) \\
        & = 2 \epsilon_{t-1} \frac{\psi_{t-1}^\top \Hop[t-1]^{-1}S_{t-1}}{1 + \widetilde{w}_{t-1}^2} \eqqcolon \ell_{t}.
    \end{align}
    Consider now the event $\mathcal{E}_{t} = \mathds{1}\{0 \le s \le t \,:\, Z_s \le \beta_t\}$, being $\beta_t$ a non-negative non-decreasing predictable process, whose expression will be defined later. Furthermore, let us define $\widetilde{\beta}_t = \min\left\{ \beta_{t}, \frac{(t-1)RK}{\sqrt{\lambda}}\right\}$ which is non-decreasing as well. Under event $\mathcal{E}_{t}$, we know that $Z_s \le \widetilde{\beta}_t$ thanks to Lemma~\ref{lemma:tech}. Under $\mathcal{E}_{t}$, we bound the maximum value and the variance of $ \ell_{t}$. Let us start with the maximum value:
    \begin{align}
        \ell_{t} \mathcal{E}_{t} \le | \ell_{t} \mathcal{E}_{t}| & \le \left|2 \epsilon_{t-1} \frac{\psi_{t-1}^\top \Hop[t-1]^{-1}S_{t-1}}{1 + \widetilde{w}_{t-1}^2} \mathcal{E}_{t} \right| \\
        & \le \frac{2 R}{1 + \widetilde{w}_{t-1}^2} \|\psi_{t-1}\|_{\Hop[t-1]^{-1}}  \|S_{t-1}\|_{\Hop[t-1]^{-1}}\mathcal{E}_{t} \label{line:-1000}\\
        & \le \frac{2 R \|\psi_{t-1}\|_{\lambda^{-1} I} \beta_{t-1}  }{1 + {\widetilde{w}_{t-1}^2}}  \label{line:-2000}\\
        & \le \frac{2 R K}{\sqrt{\lambda}} \beta_t,\label{line:-3000}
    \end{align}
    where line~\eqref{line:-1000} follows from the application of Cauchy-Schwarz inequality and recalling that $|\epsilon_{t-1}|\le R$ a.s., line~\eqref{line:-2000} is obtained by observing that $\Hop[t-1] \succeq \lambda I$ and by exploiting event $\mathcal{E}_{t}$, and line~\eqref{line:-3000} comes from the bound on  $\|\psi_{t-1}\| \le K$ and the monotonicity of $\beta_t$.
    Let us move to the variance, recalling that $\ell_t$ is zero mean, i.e., $\E[\ell_t|\mathcal{F}_{t-1}]=0$:
    \begin{align}
        \E\left[ \ell_{t}^2 | \mathcal{F}_{t-1} \right] & = \E\left[ \left( 2 \epsilon_{t-1} \frac{\psi_{t-1}^\top \Hop[t-1]^{-1}S_{t-1}}{1 + \widetilde{w}_{t-1}^2} \right)^2  \mathcal{E}_{t} \middle| \mathcal{F}_{t-1} \right] \\
        & \le   \frac{4\sigma_{t-1}^2\|\psi_{t-1}\|_{\Hop[t-1]^{-1}}^2 \|S_{t-1}\|_{\Hop[t-1]^{-1}}^2}{(1+\widetilde{w}_{t-1}^2)^2}  \mathcal{E}_{t} \label{line:-3500} \\
        & \le  \left( \frac{2\widetilde{w}_{t-1}}{1+\widetilde{w}_{t-1}^2}\right)^2 \widetilde{\beta}_{t-1}^2  \le  \min\{1,2\widetilde{w}_{t-1}\}^2 \widetilde{\beta}_{t-1}^2, \label{line:-4000}
    \end{align}
    where line~\eqref{line:-3500} follows from Cauchy-Schwarz inequality and recalling that $\E[\epsilon_{t-1}|\mathcal{F}_{t-1}]= \sigma^2_{t-1}$, line~\eqref{line:-4000} follows from the inequality $\frac{2x}{1+x^2} \le \min\{1,2x\}$ for $x \ge 0$. Summing over $\dsb{1, t-1}$, we obtain:
    \begin{align}
        \sum_{s=1}^{t-1} \E\left[ \ell_{s+1}^2 | \mathcal{F}_{s-1} \right] & \le \sum_{s=1}^{t-1} \min\{1,2\widetilde{w}_{s}\}^2 \widetilde{\beta}_{s}^2  \le 4 \widetilde{\beta}_t^2 \sum_{s=1}^{t-1}  \min\{1,\widetilde{w}_{s}\}^2,
    \end{align}
    where we bounded $\widetilde{\beta}_{s} \le \widetilde{\beta}_t$ and $ \min\{1,2\widetilde{w}_{s}\}^2 \le 4  \min\{1,\widetilde{w}_{s}\}^2$.
    From the elliptical potential lemma (Lemma~\ref{lemma:elliptic} with $M=1$), we obtain:
    \begin{align}
        \sum_{s=1}^{t-1}  \min\{1,\widetilde{w}_{s}\}^2 \le  2 \log \det(\lambda^{-1}\Hop),
    \end{align}
    where $\det (\cdot)$ is the Fredholm determinant.
    By Theorem~\ref{thm:newFreedman}, setting $\eta = e$, $v_0 = 1$, $v_t =  8 {\beta}_t^2 \log \det (\lambda^{-1} \Hop)$, we have that with probability at least $1-\delta$, simultaneously for all $t \ge 1$:
    \begin{align}
        \sum_{s=1}^t \ell_{s} \le  \sqrt{2 \max\left\{1, 8e \beta_t^2 \log \det (\lambda^{-1} \Hop) \right\}\log \frac{\pi^2(\widehat{\rho}_t+1)^2}{6\delta}} + \frac{2RK}{3\sqrt{\lambda}} \beta_t \log \frac{\pi^2(\widehat{\rho}_t+1)^2}{6\delta},
    \end{align}
    with $\widehat{\rho}_t = \max\left\{ 0 , \left\lceil \log \left( 8 \frac{(t-1)^2R^2K^2}{\lambda} \log \det (\lambda^{-1} \Hop) \right)  \right\rceil\right\}$,
    having bounded $ \widetilde{\beta}_t \le \beta_t$ in the inequality and $ \widetilde{\beta}_t \le \frac{(t-1)RK}{\sqrt{\lambda}}$ in the expression of $\widehat{\rho}_t$.
    
    \paragraph{Analysis of Term (B).}
    We proceed again with the Sherman-Morris formula in Hilbert spaces:
    \begin{align}
        \epsilon_{t-1}^2 \psi_{t-1}^\top   \Hop^{-1} \psi_{t-1} & = \epsilon_{t-1}^2 \Bigg(  \psi_{t-1}^\top   \Hop[t-1]^{-1} \psi_{t-1}  - \frac{\psi_{t-1}^\top   \Hop[t-1]^{-1} \psi_{t-1} \psi_{t-1}^\top   \Hop[t-1]^{-1} \psi_{t-1} \sigma^2_{t-1}}{1+\widetilde{w}_{t-1}^2} \Bigg) \\
        & ={\frac{ \epsilon_{t-1}^2 \|\psi_{t-1}\|_{\Hop[t-1]^{-1}}^2} {1+\widetilde{w}_{t-1}^2}}.
    \end{align}
    Let us define: $\ell_{t} \coloneqq \frac{ \epsilon_{t-1}^2 \|\psi_{t-1}\|_{\Hop[t-1]^{-1}}^2} {1+\widetilde{w}_{t-1}^2} - \E\left[  \frac{ \epsilon_{t-1}^2 \|\psi_{t-1}\|_{\Hop[t-1]^{-1}}^2} {1+\widetilde{w}_{t-1}^2} |\mathcal{F}_{t-1} \right]$ and let us start bounding the maximum value:
    \begin{align}\label{eq:boundInfNorm1}
        \ell_{t} \le  \frac{ \epsilon_{t-1}^2 \|\psi_{t-1}\|_{\Hop[t-1]^{-1}}^2} {1+\widetilde{w}_{t-1}^2}  \le \frac{R^2K^2}{\lambda},
    \end{align}
    where we bounded $\|\psi_{t-1}\|_{\Hop[t-1]^{-1}}^2 \le \|\psi_{t-1}\|_{\lambda^{-1}I}^2 \le \frac{K^2}{\lambda}$. 
    
    Concerning the variance, we have:
    \begin{align}
        \Var[\ell_t | \mathcal{F}_{t-1} ] & = \Var\left[ \frac{ \epsilon_{t-1}^2 \|\psi_{t-1}\|_{\Hop[t-1]^{-1}}^2} {1+\widetilde{w}_{t-1}^2}  \middle| \mathcal{F}_{t-1} \right] \le \E\left[ \left( \frac{ \epsilon_{t-1}^2 \|\psi_{t-1}\|_{\Hop[t-1]^{-1}}^2} {1+\widetilde{w}_{t-1}^2} \right)^2 \middle| \mathcal{F}_{t-1} \right] \\
        & \le \frac{R^2K^2}{\lambda} \E\left[  \frac{ \epsilon_{t-1}^2 \|\psi_{t-1}\|_{\Hop[t-1]^{-1}}^2} {1+\widetilde{w}_{t-1}^2}  \middle| \mathcal{F}_{t-1} \right]\label{line:n-1} = \frac{R^2K^2}{\lambda}   \frac{ \sigma_{t-1}^2 \|\psi_{t-1}\|_{\Hop[t-1]^{-1}}^2} {1+\widetilde{w}_{t-1}^2} \\
        & = \frac{R^2K^2}{\lambda}   \frac{ \widetilde{w}_{t-1}^2} {1+\widetilde{w}_{t-1}^2}  \le \frac{R^2K^2}{\lambda}   \min\{1,\widetilde{w}_{t-1}\}^2, \label{line:-6000}
    \end{align}
    where line~\eqref{line:n-1} derives from applying Equation~\eqref{eq:boundInfNorm1}, line~\eqref{line:-6000} follows from the inequality $\frac{x}{1+x} \le \min\{1,x\}$. Summing and applying the elliptic potential lemma (Lemma~\ref{lemma:elliptic} with $M=1$), we have:
    \begin{align}
        \sum_{s=1}^{t-1} \Var[\ell_{s+1} | \mathcal{F}_{s-1} ] \le \frac{R^2K^2}{\lambda}  \sum_{s=1}^{t-1}   \min\{1,\widetilde{w}_{s}\}^2 \le \frac{2R^2K^2}{\lambda}  \log \det(\lambda^{-1}\Hop).
    \end{align}
    Furthermore, following the same steps from Equation \eqref{line:n-1}, we obtain:
    \begin{align}
        \sum_{s=1}^{t-1} \E[  \ell_{s+1}|\mathcal{F}_{s-1}] = \sum_{s=1}^{t-1} \E\left[ \frac{ \epsilon_{s}^2 \|\psi_{s}\|_{\Hop[t-1]^{-1}}^2} {1+\widetilde{w}_{s}^2} \middle| \mathcal{F}_{s-1}  \right] \le 2 \log \det(\lambda^{-1}\Hop).
    \end{align}
    We now apply Theorem~\ref{thm:newFreedman}, setting $\eta = e$, $v_0 = 1$, $v_t =  \frac{2R^2K^2}{\lambda}  \log \det(\lambda^{-1}\Hop)$, we have that with probability at least $1-\delta$, simultaneously for all $t \ge 1$:
    \begin{align}
       & \sum_{s=1}^{t-1}  \frac{ \epsilon_{s}^2 \|\psi_{s}\|_{\Hop[s]^{-1}}^2} {1+\widetilde{w}_{s}^2}  \le  2 \log \det(\lambda^{-1}\Hop)\\
        &  \; + \sqrt{2 \max\left\{1, \frac{2e R^2K^2}{\lambda}  \log \det(\lambda^{-1}\Hop) \right\}\log \frac{\pi^2(\widetilde{\rho}_t+1)^2}{6\delta}} + \frac{R^2K^2}{3{\lambda}} \log \frac{\pi^2(\widetilde{\rho}_t+1)^2}{6\delta},
    \end{align}
    with $\widetilde{\rho}_t = \max\left\{ 0 , \left\lceil \log \left(  \frac{2R^2K^2}{\lambda}  \log \det(\lambda^{-1}\Hop) \right)  \right\rceil\right\}$.
    
    \paragraph{Putting All Together.}
    Since $\widehat{\rho}_t \ge \widetilde{\rho}_t$ and $\log\det(\lambda^{-1}\Hop) \le (t-1) \log \left(1 + \frac{R^2K^2}{\lambda} \right)$ from Lemma~\ref{lemma:tech}, we define $\rho_t \coloneqq \max\left\{ 0 , \left\lceil \log \left(  \frac{8R^2K^2(t-1)^3}{\lambda}   \log \left(1 + \frac{K^2R^2}{\lambda} \right) \right)  \right\rceil\right\}$.  Putting together the two bounds, we have to find $\beta_t$ in order to satisfy the following condition:
    {\thinmuskip=2mu
    \medmuskip=2mu
    \thickmuskip=2mu
    \begin{align}
        \text{(A)} + \text{(B)} \leq & \sqrt{2 \max\left\{1, 8e \beta_t^2 \log \det (\lambda^{-1} \Hop) \right\}\log \frac{\pi^2({\rho}_t+1)^2}{6\delta}} + \frac{2RK}{3\sqrt{\lambda}} \beta_t \log \frac{\pi^2({\rho}_t+1)^2}{6\delta}  \\
        & + 2 \log \det(\lambda^{-1}\Hop) + \sqrt{2 \max\left\{1, \frac{2eR^2K^2}{\lambda}  \log \det(\lambda^{-1}\Hop) \right\}\log \frac{\pi^2({\rho}_t+1)^2}{6\delta}}\\
        & + \frac{R^2K^2}{3{\lambda}} \log \frac{\pi^2({\rho}_t+1)^2}{6\delta} \le \beta_t^2.
    \end{align}
    We proceed by bounding the maxima in the left-hand-side as $\max\{a,b\} \le a+b$ for $a,b \ge 0$ and using the subadditivity of the square root to get a stricter condition:
    \begin{align}
        &    \sqrt{2\log \frac{\pi^2({\rho}_t+1)^2}{6\delta}} + \sqrt{ 16e \beta_t^2 \log \det (\lambda^{-1} \Hop) \log \frac{\pi^2({\rho}_t+1)^2}{6\delta}} + \frac{2RK}{3\sqrt{\lambda}} \beta_t \log \frac{\pi^2({\rho}_t+1)^2}{6\delta}  \\
        & \quad +2 \log \det(\lambda^{-1}\Hop) +   \sqrt{2\log \frac{\pi^2({\rho}_t+1)^2}{6\delta}} + \sqrt{ \frac{4eR^2K^2}{\lambda}  \log \det(\lambda^{-1}\Hop) \log \frac{\pi^2({\rho}_t+1)^2}{6\delta}}\\
        & \quad + \frac{R^2K^2}{3{\lambda}} \log \frac{\pi^2({\rho}_t+1)^2}{6\delta} \le \beta_t^2.
    \end{align}}%
    This is a second-degree inequality in the variable $\beta_t$ and, thus, we have to find the minimum value of $\beta_t$ fulfilling such an inequality. Using the polynomial inequality of Proposition 7 of \citep{abeille2021instance} (i.e., $x^2 \le bx+c=0 \implies x \le b + \sqrt{c}$ when $b, c \ge 0$), we have:
    \begin{align}
        \beta_t & \le \sqrt{ 16e  \log \det (\lambda^{-1} \Hop) \log \frac{\pi^2({\rho}_t+1)^2}{6\delta}} +\frac{2RK}{3\sqrt{\lambda}} \log \frac{\pi^2({\rho}_t+1)^2}{6\delta} \\
        & \quad + \Bigg( 2\sqrt{2\log \frac{\pi^2({\rho}_t+1)^2}{6\delta}} + 2 \log \det(\lambda^{-1}\Hop) \\
        & \quad + \sqrt{ \frac{4eR^2K^2}{\lambda}  \log \det(\lambda^{-1}\Hop) \log \frac{\pi^2({\rho}_t+1)^2}{6\delta}} +  \frac{R^2K^2}{3{\lambda}} \log \frac{\pi^2({\rho}_t+1)^2}{6\delta} \Bigg)^{\frac{1}{2}} \\
        & \le  \left( (\sqrt{ 16e } + \sqrt{2})  \sqrt{\log \det (\lambda^{-1} \Hop)} + \sqrt{2\sqrt{2}} \right) \sqrt{ \log \frac{\pi^2({\rho}_t+1)^2}{6\delta}} \\
        & \quad + \left( \frac{2}{3} + \frac{1}{\sqrt{3}} \right) \frac{RK}{\sqrt{\lambda}} \log \frac{\pi^2({\rho}_t+1)^2}{6\delta} \label{line:n001} \\
        & \quad +\left( \frac{4eR^2K^2}{\lambda}  \log \det(\lambda^{-1}\Hop) \log \frac{\pi^2({\rho}_t+1)^2}{6\delta}\right)^{\frac{1}{4}}\\
         & \le  \left( \left(\sqrt{ 16e } + \sqrt{2}+ \frac{1}{2}\right)  \sqrt{\log \det (\lambda^{-1} \Hop)} + \sqrt{2\sqrt{2}} \right) \sqrt{ \log \frac{\pi^2({\rho}_t+1)^2}{6\delta}} \\
        & \quad + \left( \frac{2}{3} + \frac{1}{\sqrt{3}} + \sqrt{e} \right) \frac{RK}{\sqrt{\lambda}} \log \frac{\pi^2({\rho}_t+1)^2}{6\delta}, \label{line:n002}
    \end{align}
    where line~\eqref{line:n001} follows from the subadditivity of the square root and recalling that $\log \frac{6({\rho}_t+1)^2}{\pi^2\delta} \ge 1$ for $t \ge 1$, to get line~\eqref{line:n002}, we apply Young's inequality for products as $ab \le a^2/2 + b^2/2$ for $a,b \ge 0$ to get:
    \begin{equation}\resizebox{.92\textwidth}{!}{$\displaystyle
          \left(\frac{4eR^2K^2}{\lambda}  \log \det(\lambda^{-1}\Hop) \log \frac{\pi^2({\rho}_t+1)^2}{6\delta}\right)^{\frac{1}{4}} \le \sqrt{e} \frac{RK}{\sqrt{\lambda}} \sqrt{\log \frac{\pi^2({\rho}_t+1)^2}{6\delta}} + \frac{1}{2} \sqrt{ \log \det(\lambda^{-1}\Hop) }.$}
    \end{equation}
    To obtain more manageable constant, we write:
    \begin{align}
        \beta_t & \le  \left( \sqrt{ 73 \log\det(\lambda^{-1}\Hop) } + \sqrt{3} \right)\sqrt{\log \frac{\pi^2({\rho}_t+1)^2}{6\delta}} +\frac{3RK}{ \sqrt{\lambda}} \log \frac{\pi^2({\rho}_t+1)^2}{6\delta}.
    \end{align}
   A simple inductive argument allows to conclude that, with probability at least $1-2\delta$:
    \begin{align}
        Z_t^2 \le  \left( \sqrt{ 73 \log\det(\lambda^{-1}\Hop) } + \sqrt{3} \right)\sqrt{\log \frac{\pi^2({\rho}_t+1)^2}{6\delta}} +\frac{3RK}{ \sqrt{\lambda}} \log \frac{\pi^2({\rho}_t+1)^2}{6\delta}.
    \end{align}
    Notice that, as requested, $\beta_t$ is a non-decreasing sequence of $t$, since $\rho_t$ is non-decreasing with $t$ and $\det(\lambda^{-1}\Hop)$ is non-decreasing as well. Indeed, since $\Hop = \Hop[t-1] + \sigma_{t-1} \psi_{t-1} \psi_{t-1}^\top =\widetilde{V}_{t-1}^{1/2} (I +  \sigma_{t-1} \widetilde{V}_{t-1}^{-1/2} \psi_{t-1} \psi_{t-1}^\top \widetilde{V}_{t-1}^{-1/2}) \widetilde{V}_{t-1}^{1/2}$ and $\sigma_{t-1}\widetilde{V}_{t-1}^{-1/2}  \psi_{t-1} \psi_{t-1}^\top \widetilde{V}_{t-1}^{-1/2}$ has rank one (and, thus, it is trace-class), we have, from the  Weinstein-Aronszajn identity \citep{kato2013perturbation} that $\det(\lambda^{-1}\Hop) = \det(\lambda^{-1}\Hop[t-1]) (1 + \sigma_{t-1}\|\psi_{t-1}\|^2_{\widetilde{V}_{t-1}^{-1}}) \ge \det(\lambda^{-1}\Hop[t-1])$ (see also the proof of Lemma \ref{lemma:elliptic}). Rescaling $\delta \leftarrow \delta/2$, we get the result.
\end{proof}

\subsection{Proofs of Section~\ref{sec:regret}}
\goodEvent*
\begin{proof}
        First of all, we observe that $\mathcal{E}_\delta = \left\{ \forall t \ge 1 \,:\, \left\| g_t(f^*) - g_t(\widehat{f}_t) \right\|_{\Hop^{-1}(\lambda;f)} \le B_t(\delta;f) \right\}$. Let $t \in \Nat$, we have:
    \begin{align}
          g_t(f^*) & - g_t(\widehat{f}_t)  \\
        & =  \sum_{s=1}^{t-1} \gtau^{-1}\mu(f^*(\xs_s)) \phi(\xs_s) + \lambda f^* -  \sum_{s=1}^{t-1} \gtau^{-1} \mu(\widehat{f}_r(\xs_s)) \phi(\xs_s) - \lambda \widehat{f}_t \\
        & =   \sum_{s=1}^{t-1} \gtau^{-1}(-y_s  + \mu(f^*(\xs_s)) \phi(\xs_s)) + \lambda f^*  \\
        & \quad -\Bigg(\underbrace{  \sum_{s=1}^{t-1} \gtau^{-1}(-y_s + \mu(\widehat{f}_t(\xs_s))) \phi(\xs_s) + \lambda \widehat{f}_t}_{\nabla \mathcal{L}_t(\widehat{f}_t) = 0}\Bigg) \\
        & =   - \gtau^{-1} \sum_{s=1}^{t-1}\epsilon_s\phi(\xs_s) + \lambda f^*,
        \end{align}
        having exploited the first-order optimality condition for the loss evaluated in the maximum-likelihood estimate, i.e., $\nabla \mathcal{L}_t(\widehat{f}_t) = 0$ and the definition of $\epsilon_s = y_s - \mu(f^*(\xs_s))$ for $s \in \Nat$.
    Now, by computing the norm, we have:    
    \begin{align}
        \left\| g_t(f^*)  - g_t(\widehat{f}_t) \right\|_{\Hop^{-1}(\lambda; f^*)} \le \gtau^{-1} \left\|  \sum_{s=1}^{t-1}\epsilon_s\phi(\xs_s)\right\|_{\Hop^{-1}(\lambda; f^*)} +   \lambda  \left\| f^* \right\|_{\Hop^{-1}(\lambda; f^*)}.
    \end{align}
    We can immediately bound the second term under Assumption~\ref{ass:boundedNorm} as
        $\left\| f^* \right\|_{\Hop^{-1}(\lambda; f^*)}^2 = \inner{f^*}{ \Hop^{-1}(\lambda; f^*) f^*} \le \lambda^{-1} \|f^*\|^2 \le \lambda^{-1} B^2$,
    since $\Hop^{-1}(\lambda; f^*) \succeq \lambda I$. For the first term, we resort to the self-normalized concentration inequality of Theorem~\ref{thr:selfNormalizedComplex}, recalling that the variance of the reward is $\Var[\epsilon_s|\mathcal{F}_{s-1}] = \dot{\mu}(f^*(\xs_s)) \gtau^{-1}$.
\end{proof}

\regretBound*

\begin{proof}
    We start by performing a second-order Taylor's expansion of the regret:
    \begin{align}
        \sum_{t=1}^T ( {\mu}(f^*(\xs^*)) &- {\mu}(f^*(\xs_t)) )  =  \underbrace{\sum_{t=1}^T \dot{\mu}(f^*(\xs_t)) \left( f^*(\xs^*) - f^*(\xs_t) \right)}_{\eqqcolon R_1(T)} \\
        & + \underbrace{\sum_{t=1}^T \left( \int_{v=0}^1 (1-v) \ddot{\mu}((1-v) f^*(\xs_t) + v f^*(\xs^*)) \de v\right) \left( f^*(\xs^*) - f^*(\xs_t) \right)^2}_{\eqqcolon R_2(T)}.
    \end{align}
We know that $\ftilde \in \Cset$ and, under the good event $\mathcal{E}_\delta$, we also have $f^* \in \Cset$. Using the optimism, we know that $ \widetilde{f}_t(\xs_t) \ge f^*(\xs^*)$. We start by analyzing $R_1(T)$, recalling that $\dot{\mu}(f^*(\xs_t)) \ge 0$:
\begin{align}
    R_1(T) & = \sum_{t=1}^T \dot{\mu}(f^*(\xs_t)) \left( f^*(\xs^*) - f^*(\xs_t) \right)  = \sum_{t=1}^T  \dot{\mu}(f^*(\xs_t)) \left( f^*(\xs^*) - f^*(\xs_t) \pm \widetilde{f}_t(\xs_t) \right) \\
    & \le  \sum_{t=1}^T \dot{\mu}(f^*(\xs_t)) (\widetilde{f}_t(\xs_t) - f^*(\xs_t))  =  \sum_{t=1}^T  \dot{\mu}(f^*(\xs_t)) \inner{\widetilde{f}_t - f^*}{\phi(\xs_t)} \\
    & \le  \sum_{t=1}^T  \underbrace{\|\widetilde{f}_t - f^*\|_{\Hop(\lambda; f^*)}}_{\text{(a)}} \underbrace{\dot{\mu}(f^*(\xs_t))\|\phi(\xs_t)\|_{\Hop^{-1}(\lambda; f^*)}}_{\text{(b)}},
\end{align}
where we used the reproducing property and the Cauchy-Schwarz's inequality. For term (a), we apply Lemma~\ref{lemma:relationAlphaG} with $f\leftarrow \widetilde{f}_t,f' \leftarrow f^*,f''\leftarrow \widehat{f}_t$ and exploit the good event:
\begin{align}
    \|\widetilde{f}_t - f^*\|_{\Hop(\lambda; f^*)} & \le(1+2R_sB{K})  \\
    & \quad \cdot \left( \left\| g_t(\widetilde{f}_t) - g_t( \widehat{f}_t)\right\|_{\Hop^{-1}(\lambda; \widetilde{f}_t)} + \left\| g_t(f^*) - g_t(\widehat{f}_t)\right\|_{\Hop^{-1}(\lambda; f^*)} \right) \\
    & \le (1+2R_sB{K}) (B_t(\delta;\widetilde{f}_t) + B_t(\delta;f^*) ) \\
    &\le 2(1+2R_sB{K}) {\beta}_T(\delta;\Hs) ,
\end{align}
having observed that ${\beta}_T(\delta;\Hs) \ge {\beta}_t(\delta;\Hs) \ge \max\{B_t(\delta;\widetilde{f}_t),B_t(\delta;f^*)\}$.
For term (b), we apply Cauchy-Schwarz's inequality:
\begin{align}
 \sum_{t=1}^T  & \dot{\mu}(f^*(\xs_t))\|\phi(\xs_t)\|_{\Hop^{-1}(\lambda; f^*)}   \\
 & \quad \le \sqrt{\gtau}\sqrt{\sum_{t=1}^T  \dot{\mu}(f^*(\xs_t))} \sqrt{\sum_{t=1}^T \gtau^{-1} \dot{\mu}(f^*(\xs_t)) \|\phi(\xs_t)\|_{\Hop^{-1}(\lambda; f^*)}^2 }.
\end{align}
Recalling that $\gtau^{-1} \dot{\mu}(f^*(\xs_t)) \|\phi(\xs_t)\|_{\Hop^{-1}(\lambda; f^*)}^2 = \|\widetilde{\phi}(\xs_t;f^*)\|_{\Hop^{-1}(\lambda; f^*)}^2$, we can apply an elliptic potential lemma (Lemma~\ref{lemma:elliptic} with $M = \max \left\{1, \lambda^{-1} \gtau^{-1}R_{\dot{\mu}} K^2\right\}$), where $\lambda^{-1} \gtau^{-1}R_{\dot{\mu}} K^2$ is a bound to the maximum value $\|\widetilde{\phi}(\xs_t;f^*)\|_{\Hop^{-1}(\lambda; f^*)}^2$ can take as:
\begin{align}
    \|\widetilde{\phi}(\xs_t;f^*)\|_{\Hop^{-1}(\lambda; f^*)}^2 = \gtau^{-1} \dot{\mu}(f^*(\xs_t)) \|\phi(\xs_t)\|_{\Hop^{-1}(\lambda; f^*)}^2 \le \gtau^{-1} R_{\dot{\mu}} K^2 \lambda^{-1},
\end{align}
as $\Hop^{-1}(\lambda; f^*) \succeq \lambda I$.
Thus, we have:
\begin{align}
    \sum_{t=1}^T \|\widetilde{\phi}(\xs_t;f^*)\|_{\Hop^{-1}(f^*)}^2 & \le 2 \max \left\{1, \lambda^{-1} \gtau^{-1}R_{\dot{\mu}} K^2\right\}\log \det (\lambda^{-1}\Hop(f^*)) \\
    & \le 4 \max \left\{1, \lambda^{-1} \gtau^{-1}R_{\dot{\mu}} K^2\right\} \widetilde{\gamma}_T(f^*).
\end{align}
The remaining term can be treated as follows, by means of a Taylor expansion:
    {\thinmuskip=1mu
\medmuskip=1mu
\thickmuskip=1mu
\begin{align}
    \sum_{t=1}^T & \dot{\mu}(f^*(\xs_t)) = {\sum_{t=1}^T  \dot{\mu}(f^*(\xs^*))} + \sum_{t=1}^T \left( \int_{v=0}^1 \ddot{\mu}((1-v) f^*(\xs^*) + v f^*(\xs_t)) \de v \right) (f^*(\xs_t) - f^*(\xs^*)) \\
    & \le T \dot{\mu}(f^*(\xs^*)) +  R_s \sum_{t=1}^T \left( \int_{v=0}^1 \dot{\mu}((1-v) f^*(\xs^*) + v f^*(\xs_t)) \de v  \right) (f^*(\xs^*) - f^*(\xs_t) ) \\
    & = \frac{T}{ \kappa_*} + R_s \sum_{t=1}^T (\mu(f^*(\xs^*)) - \mu(f^*(\xs_t))) \\
    & = \frac{T}{ \kappa_*} + R_s  R(\text{\algnameshort},T) = \frac{T}{ \kappa_*} + R_s  R(\text{\algnameshort},T).
\end{align}}%
where we exploited $f^*(\xs^*) \ge f^*(\xs_t)$, the self-concordance property (Assumption~\ref{asm:selfConc}) and mean-value theorem. Putting all together, we get:
\begin{align}
    R_1(T) & \le 4 \sqrt{\gtau} (1+2R_sB{K}) {\beta}_T(\delta;\Hs) \sqrt{\frac{T}{ \kappa_*} + R_s  R(\text{\algnameshort},T)} \sqrt{ \max \left\{1, \frac{R_{\dot{\mu}} K^2}{\lambda \, \gtau}\right\} \widetilde{\gamma}_T(f^*)}.
\end{align}
Let us move to the second term, using optimism and proceeding with the same rationale:
\begin{align}
    R_2(T) & \le R_{\dot{\mu}} R_s \sum_{t=1}^T (f^*(\xs^*) - f^*(\xs_t))^2  \le R_{\dot{\mu}} R_s \sum_{t=1}^T (\widetilde{f}_t(\xs_t) - f^*(\xs_t))^2 \\
    & \le R_{\dot{\mu}} R_s \sum_{t=1}^T  {\|\widetilde{f}_t - f^*\|_{\Hop(\lambda; f^*)}^2} {\|\phi(\xs_t)\|_{\Hop^{-1}(\lambda;f^*)}^2} \\
    & \le 4 R_{\dot{\mu}}  R_s(1+2R_sB{K})^2 {\beta}_T(\delta;\Hs)^2 \sum_{t=1}^T \|\phi(\xs_t)\|_{\Hop^{-1}(\lambda;f^*)}^2 \\
    & \le 4 \gtau R_{\dot{\mu}} R_s \kappa_{\Xs} (1+2R_sB{K})^2 {\beta}_T(\delta;\Hs)^2 \sum_{t=1}^T \gtau^{-1} \dot{\mu}(f^*(\xs_t)) \|\phi(\xs_t)\|_{\Hop^{-1}(\lambda; f^*)}^2 \\
    & \le 4 \gtau R_{\dot{\mu}} R_s \kappa_{\Xs} (1+2R_sB{K})^2 {\beta}_T(\delta;\Hs)^2 \sum_{t=1}^T \|\widetilde{\phi}(\xs_t;f^*)\|_{\Hop^{-1}(\lambda; f^*)}^2 \\
    & \le 16 \gtau R_{\dot{\mu}}  R_s \kappa_{\Xs} (1+2R_sB{K})^2 {\beta}_T(\delta;\Hs)^2 \max \left\{1, \lambda^{-1} \gtau^{-1}R_{\dot{\mu}} K^2\right\} \widetilde{\gamma}_T(f^*),
\end{align}
having, in addition, exploited the fact that $ \kappa_{\Xs} \ge \dot{\mu}(f^*(\xs_t))^{-1}$.
Putting all together, we have:
\begin{align}
R(\text{\algnameshort},T) & = R_1(T) + R_2(T) \\
    & \le 4 \sqrt{\gtau} (1+2R_sB{K}){\beta}_T(\delta;\Hs) \sqrt{\max \left\{1, \lambda^{-1} \gtau^{-1}R_{\dot{\mu}} K^2\right\}\widetilde{\gamma}_T(f^*)} \\
    & \quad \cdot \left(\sqrt{\frac{T}{\kappa_{*}}} + \sqrt{R_s R(\text{\algnameshort},T)} \right) + R_2(T). 
\end{align}
 Using the polynomial inequality of Proposition 7 of \citep{abeille2021instance} (i.e., $x^2 \le bx+c=0 \implies x \le b + \sqrt{c}$ when $b, c \ge 0$), we have:
\begin{align*}
    \sqrt{R(\text{\algnameshort},T)} \le  4 \sqrt{\gtau} (1+2R_sB{K}){\beta}_T(\delta;\Hs) \sqrt{\max \left\{1, \lambda^{-1} \gtau^{-1}R_{\dot{\mu}}K^2\right\}\widetilde{\gamma}_T(f^*)} \sqrt{R_s}  \\
    + \sqrt{4 \sqrt{\gtau} (1+2R_sB{K}){\beta}_T(\delta;\Hs) \sqrt{\max \left\{1, \lambda^{-1} \gtau^{-1}R_{\dot{\mu}}K^2\right\} \widetilde{\gamma}_T(f^*)} \sqrt{\frac{T}{\kappa_{*}}} + R_2(T) }.
\end{align*}
Squaring both sides and bounding the square as $(a+b)^2 \le 2a^2+2b^2$, we obtain:
    {\thinmuskip=1mu
\medmuskip=1mu
\thickmuskip=1mu
\begin{align*}
    R(&\text{\algnameshort},T) \le 2\left(4 \sqrt{\gtau} (1+2R_sB{K}){\beta}_T(\delta;\Hs) \sqrt{\max \left\{1, \lambda^{-1} \gtau^{-1}R_{\dot{\mu}} K^2\right\}\widetilde{\gamma}_T(f^*)} \sqrt{R_s} \right)^2 \\
    & \ \ \ + 2\left({4 \sqrt{\gtau} (1+2R_sB{K}){\beta}_T(\delta;\Hs) \sqrt{\max \left\{1, \lambda^{-1} \gtau^{-1}R_{\dot{\mu}}K^2\right\}\widetilde{\gamma}_T(f^*)} \sqrt{\frac{T}{\kappa_{*}}} + R_2(T) }\right) \\
    & \le 8 \sqrt{\gtau} (1+2R_sB{K}){\beta}_T(\delta;\Hs) \sqrt{\max \left\{1, \lambda^{-1} \gtau^{-1}R_{\dot{\mu}} K^2\right\}\widetilde{\gamma}_T(f^*)} \sqrt{\frac{T}{\kappa_{*}}} \\
    & \ \ \ + 32  R_s (1+ R_{\dot{\mu}} \kappa_{\Xs}) \gtau (1+2R_sBK)^2{\beta}_T(\delta;\Hs)^2 \max \left\{1, \lambda^{-1} \gtau^{-1}R_{\dot{\mu}} K^2\right\} \widetilde{\gamma}_T(f^*).
\end{align*}}%
We get the result by defining $  R_{\textnormal{\text{perm}}}(T)$ and $R_{\textnormal{\text{trans}}}(T)$ as in the statement.
\end{proof}

\subsection{Proofs of Section~\ref{sec:efficient}}

\begin{restatable}[Confidence Set]{lemm}{cSetSecond}\label{lemma:cSetSecond}
Let $t \in \Nat$, $f \in \Hs$, and $\delta \in (0,1)$. Then, it holds that $\Cset \subseteq \mathcal{D}_t(\delta)$. Furthermore, under the good event $\mathcal{E}_\delta$, for every $f \in \mathcal{D}_t(\delta)$, we have:
\begin{align}
    \| f - f^*\|_{\Hop(\lambda; f^*)} \le (2+2R_sBK) \sqrt{{\beta}_t'(\delta)} \left( \sqrt{{\beta}_t'(\delta)} + \sqrt{2} \right).
\end{align}
\end{restatable}
\begin{proof}
    Proceeding as in Lemma 2 of \citep{abeille2021instance}, being $\eL(\cdot)$ twice Fréchet differentiable, using Taylor's expansion the definitions of $G_t$ and $\widetilde{G}_t$ in Appendix~\ref{apx:technical}, we have:
    {\thinmuskip=2mu
\medmuskip=2mu
\thickmuskip=2mu
    \begin{align}
        \eL (f) &  - \eL(\fhat) \\ 
        & = \inner{f-\widehat{f}_t}{\underbrace{\nabla \eL(\fhat)}_{=0}}+ \left\langle f - \widehat{f}_t, \left( \int_{v=0}^1(1-v) \Hop(\lambda;\fhat + v(f-\fhat)) \de v \right)(f - \widehat{f}_t) \right\rangle \\
        & = \left\| f - \widehat{f}_t \right\|_{\widetilde{G}_t(\fhat,f)}^2  \le \left\| f - \widehat{f}_t \right\|_{{G}_t(\fhat,f)}^2 \\
        & = \left\| g_t(f) - g_t(\fhat) \right\|_{\widetilde{G}_t^{-1}(\fhat,f)}^2  \le (1+2R_sBK) \left\| g_t(f) - g_t(\fhat) \right\|_{\Hop^{-1}(\lambda; f)},
    \end{align}}%
    where we used Equations~\eqref{eq:mvt} and~\eqref{eq:7}.
    Thus, let $f \in \Cset$. From this, it follows that $\left\| g_t(f) - g_t(\fhat) \right\|_{\Hop^{-1}(\lambda;f)} \le B_t(\delta;f) \le B_t(\delta;\Hs) \le B_t'(\delta)$ and, consequently, $f \in \mathcal{D}_t(\delta)$. \\
For the second part, suppose the good event $\mathcal{E}_\delta$ holds and consider $f \in \mathcal{D}_t(\delta)$, we have, via Taylor's expansion:
    {\thinmuskip=2mu
\medmuskip=2mu
\thickmuskip=2mu
    \begin{align}
         \eL(f) & - \eL(f^*) \\ 
         & = \inner{f-f^*}{\nabla \eL(f^*)}+ \left\langle f - f^*, \left( \int_{v=0}^1(1-v) \Hop(\lambda;f^* + v(f-f^*)) \de v \right)(f -  f^*) \right\rangle \\
        & = \inner{f-f^*}{ \nabla \eL(f^*)}+\|f - f^*\|^2_{\widetilde{G}_t(f^*,f)} \\
         & \ge  \inner{f-f^*}{\nabla \eL(f^*)}+ (2+2R_sBK)^{-1}\|f - f^*\|^2_{\Hop(\lambda; f^*)},
    \end{align}}%
    where we used Equation~\eqref{eq:8}.
    Thus, we have:
    {\thinmuskip=2mu
\medmuskip=2mu
\thickmuskip=2mu
    \begin{align}
        \|f - &  f^*\|^2_{\Hop(\lambda;f^*)} \\
        & \le (2+2R_sBK)( \eL(f) - \eL(f^*) ) + (2+2R_sBK)\inner{f-f^*}{\nabla \eL(f^*)} \\
        & \le (2+2R_sBK) (\eL(f) - \eL(\fhat) ) + (2+2R_sBK)(\eL(f^*) - \eL(\fhat) ) \\
        & \quad +(2+2R_sBK) \| f-f^*\|_{\Hop(\lambda; f^*)} \|\nabla \eL(f^*)\|_{\Hop^{-1}(\lambda;f^*)} \\
        & \le 2(2+2R_sBK)(1+2R_sBK){B}_t(\delta;\Hs) + (2+2R_sBK) \| f-f^*\|_{\Hop(\lambda;f^*)} B_t(\delta;f^*),
        \end{align}}%
        where we used the fact that $\eL(f) \ge \eL(\fhat) $ and $ \eL(f^*) \ge \eL(\fhat)$, that $f^* \in \Cset \subseteq \mathcal{D}_t(\delta)$ under the good event and $f \in \mathcal{D}_t(\delta)$, and that:
        \begin{align}
            \|\nabla \eL(f^*)\|_{\Hop^{-1}(\lambda;f^*)} = \|g_t(f^*) - g_t(\fhat)\|_{\Hop^{-1}(\lambda;f^*)} \le B_t(\delta;f^*),
        \end{align}
        holding under the good event. By the choice of confidence radius and bounding $B_t(\delta;f^*)\le {B}_t(\delta;\Hs)\le {\beta}_t(\delta;\Hs) \le \beta'_t(\delta)$, we have the second-degree inequality:
        \begin{align}
            \|f - f^*\|^2_{\Hop(\lambda;f^*)} \le 2 & (2+2R_sBK)(1+2R_sBK){\beta}_t'(\delta) + (2+2R_sBK) \| f-f^*\|_{\Hop(\lambda;f^*)} {\beta}_t'(\delta).
        \end{align}
         Using the polynomial inequality of Proposition 7 of \citep{abeille2021instance} (i.e., $x^2 \le bx+c=0 \implies x \le b + \sqrt{c}$ when $b, c \ge 0$), we have:
        \begin{align}
            \|f - f^*\|_{\Hop(\lambda;f^*)} & \le \sqrt{2(2+2R_sBK)(1+2R_sBK){\beta}_t'(\delta)}+ (2+2R_sBK) {\beta}_t'(\delta) \\
            & \le (2+2R_sBK) \sqrt{{\beta}_t'(\delta)} \left( \sqrt{{\beta}_t'(\delta)} + \sqrt{2} \right),
        \end{align}
        having bounded $1+2R_sBK \le 2+2R_sBK$.
\end{proof}

\regretBoundSecond*
\begin{proof}
    The proof follows the same steps as Theorem~\ref{thm:regret}, with the only difference that we exploit the bound of Lemma~\ref{lemma:cSetSecond}:
    \begin{align}
       \| f - f^*\|_{\Hop(\lambda;f^*)} \le (2+2R_sBK) \sqrt{{\beta}_t'(\delta)} \left( \sqrt{{\beta}_t'(\delta)} + \sqrt{2} \right).
       \end{align}
\end{proof}

\section{Technical Lemmas}\label{apx:technical}

In this section, we introduce some technical concepts and lemmas to be used in the analysis.
We consider $\xs \in \Xs$ and $f,f' \in \mathcal{H}$ such that $f(\xs)=\inner{f}{\phi(\xs)},f'(\xs)=\inner{f'}{\phi(\xs)}$  for every $\xs \in \Xs$, we define the following quantities, analogous to those of \citep{abeille2021instance}:
\begin{align}
    & \xi(\xs,f,f') \coloneqq \int_{v=0}^1 \dot{\mu}((1-v)f(\xs) + vf'(\xs)) \de v, \\
    & \widetilde{\xi}(\xs,f,f') \coloneqq \int_{v=0}^1 (1-v)\dot{\mu}((1-v)f(\xs) + vf'(\xs)) \de v, \\
    & G_t(f,f') \coloneqq \sum_{s=1}^{t-1}\frac{ \xi(\xs_s,f,f')}{\gtau} \phi(\xs_s) \phi(\xs_s)^\top  + \lambda I = \gtau^{-1} \Phi_t^\top \bm{\Xi}_t(f,f') \Phi_t+ \lambda I,\\
    & \widetilde{G}_t(f,f') \coloneqq \sum_{s=1}^{t-1} \frac{\widetilde{\xi}(\xs_s,f,f') }{\gtau}\phi(\xs_s) \phi(\xs_s)^\top  + \lambda I = \gtau^{-1} \Phi_t^\top \widetilde{\bm{\Xi}}_t(f,f') \Phi_t+ \lambda I,
\end{align}
where $\bm{\Xi}_t(f,f') = \mathrm{diag}(({\xi}(\xs_s,f,f'))_{s \in \dsb{t-1}})$ and $\widetilde{\bm{\Xi}}_t(f,f') = \mathrm{diag}((\widetilde{\xi}(\xs_s,f,f'))_{s \in \dsb{t-1}})$ are diagonal positive definite matrices.
For every $\xs \in \Xs$, we have that $ \xi(\xs,f,f')  \ge \widetilde{\xi}(\xs,f,f')$ and, consequently, $\bm{\Xi}_t(f,f') \succeq \widetilde{\bm{\Xi}}_t(f,f')$. It follows that $ \Phi_t^\top {\bm{\Xi}}_t(f,f') \Phi_t \succeq  \Phi_t^\top \widetilde{\bm{\Xi}}_t(f,f') \Phi_t$ and, finally, $ G_t(f,f') \succeq \widetilde{G}_t(f,f')$. Thanks to the mean-value theorem and the definition of function $g_t(f)$, we have that:
\begin{align}\label{eq:mvt}
    g_t(f) - g_t(f') = G_t(f,f') (f - f').
\end{align}
Using Assumption~\ref{asm:selfConc}, we can easily extend Lemmas 7 and 8 of \citet{abeille2021instance}.

\begin{lemm}[Extension of Lemma~7 of \citealp{abeille2021instance}]\label{lemma:7}
    Let $\mathcal{Z} \subset \Reals$ be any bounded interval of $\Reals$ and let $f : \mathcal{Z} \rightarrow \Reals$ be a monotonically non-decreasing function such that $|\ddot{f}| \le R_s \dot{f}$. Then, for every $z_1,z_2 \in\mathcal{Z}$:
    \begin{align}
        \int_{v=0}^1 \dot{f}(z_1+v(z_2-z_1) )\de v \ge \frac{\dot{f}(z)}{1+R_s |z_1-z_2|}, \quad \forall z \in \{z_1,z_2\}.
    \end{align}
\end{lemm}

\begin{proof}
    Immediately follows from the same steps of \citet[][Lemma 7]{abeille2021instance}.
\end{proof}

\begin{lemm}[Extension of Lemma~8 of \citealp{abeille2021instance}]\label{lemma:8}
    Let $\mathcal{Z} \subset \Reals$ be any bounded interval of $\Reals$ and let $f : \mathcal{Z} \rightarrow \Reals$ be a monotonically non-decreasing function such that $|\ddot{f}| \le R_s \dot{f}$. Then, for every $z_1,z_2 \in\mathcal{Z}$:
    \begin{align}
        \int_{v=0}^1 (1-v)\dot{f}(z_1+v(z_2-z_1) )\de v \ge \frac{\dot{f}(z_1)}{2+R_s |z_1-z_2|}.
    \end{align}
\end{lemm}

\begin{proof}
    See \citep[][Lemma D.1]{lee2024unified}.
\end{proof}
\noindent From Lemma~\ref{lemma:7} and Lemma~\ref{lemma:8}, we immediately have for $\xs \in \Xs$ and  $\overline{f} \in \{f,f'\}$:
\begin{align}
    & \xi(\xs,f,f') \coloneqq \int_{v=0}^1  \dot{\mu}((1-v)f(\xs) + vf'(\xs)) \de v \ge \frac{\dot{\mu}(\overline{f}(\xs))}{1+R_s|f(\xs) - f'(\xs)|}, \\
    & \widetilde{\xi}(\xs,f,f') \coloneqq \int_{v=0}^1 (1-v) \dot{\mu}((1-v)f(\xs) + vf'(\xs)) \de v \ge \frac{\dot{\mu}(f(\xs))}{2+R_s|f(\xs) - f'(\xs)|}.
\end{align}
Moreover, under Assumptions~\ref{ass:boundedNorm} and~\ref{ass:boundedKernel}, we have that $|f(\xs) - f'(\xs)| \le 2 \|f\|_\infty \le 2 B K$. This allows us to conclude that  for $\overline{f} \in \{f,f'\}$, we have that $\bm{\Xi}_t(f,f') \succeq (1+2R_sBK)^{-1} \mathbf{M}_t(\overline{f})$ and $\widetilde{\bm{\Xi}}_t(f,f') \succeq (2+2R_sBK)^{-1} \mathbf{M}_t({f})$, where $\mathbf{M}_t(f) = \mathrm{diag}((\dot{\mu}(f(\xs_s)))_{t \in \dsb{t-1}})$ is a diagonal positive definite matrix and, finally, to write:
\begin{equation}\label{eq:7}
    \begin{aligned}
     G_t(f,f') & = \gtau^{-1} \Phi_t^\top \bm{\Xi}_t(f,f') \Phi_t+ \lambda I \succeq  (1+2R_sBK)^{-1} \gtau^{-1} \Phi_t^\top \mathbf{M}_t(\overline{f}) \Phi_t+ \lambda I  \\
     & \succeq (1+2R_sBK)^{-1} (\gtau^{-1} \Phi_t^\top \mathbf{M}_t(\overline{f}) \Phi_t+ \lambda I)  = (1+2R_sBK)^{-1} \Hop(\lambda;\overline{f}),
    \end{aligned}
\end{equation}
\begin{equation}\label{eq:8}
\begin{aligned}
     \widetilde{G}_t(f,f') & = \gtau^{-1} \Phi_t^\top \widetilde{\bm{\Xi}}_t(f,f') \Phi_t+ \lambda I \succeq  (2+2R_sBK)^{-1} \gtau^{-1} \Phi_t^\top \mathbf{M}_t(f) \Phi_t+ \lambda I  \\
     & \succeq (2+2R_sBK)^{-1} (\gtau^{-1} \Phi_t^\top \mathbf{M}_t({f}) \Phi_t+ \lambda I)  = (2+2R_sBK)^{-1} \Hop(\lambda;{f}),
\end{aligned}
\end{equation}
where  we observed that $\Hop(\lambda;\overline{f}) =  \gtau^{-1} \Phi_t^\top \mathbf{M}_t(\overline{f}) \Phi_t+ \lambda I $.

\begin{lemm}\label{lemma:VHbounds}
    Let $f \in \Hs$, $\Hop(\lambda;f)$ and $V_t(\lambda)$ defined as in the main paper. The following semidefinite inequalities hold:
    \begin{align}
        \min\{1,\gtau R_{\dot{\mu}}^{-1}\} \Hop(\lambda;f) \preceq V_t(\lambda) \preceq \max\{1,\gtau\kappa_{\Xs}(f)\} \Hop(\lambda;f),
    \end{align}
    where $\kappa_{\Xs}(f) = \sup_{\xs \in \Xs} \frac{1}{\dot{\mu}(f(\xs))}$.
\end{lemm}

\begin{proof}
    First of all, let us observe that for every $s \in \dsb{t-1}$, we have $ \kappa_{\Xs}(f)^{-1} \le  \dot{\mu}(f(\xs_s)) \le R_{\dot{\mu}} $. Let $ \mathbf{M}_t(f) = \mathrm{diag}((\dot{\mu}(f(\xs_s)))_{s\in \dsb{t-1}})$ which is a diagonal positive definite matrix. Moreover, the following inequalities hold:
    \begin{align}
         \kappa_{\Xs}(f)^{-1}  \Phi_t^\top  \Phi_t \! \preceq \! \lambda_{\min}(\mathbf{M}_t(f))  \Phi_t^\top  \Phi_t \preceq \Phi_t^\top \mathbf{M}_t(f) \Phi_t \preceq  \lambda_{\max}(\mathbf{M}_t(f))  \Phi_t^\top  \Phi_t \preceq  R_{\dot{\mu}} \Phi_t^\top  \Phi_t,
    \end{align}
    where $\lambda_{\max}(\cdot)$ and $\lambda_{\min}(\cdot)$ denote the maximum and minimum eigenvalues, respectively.
    By observing that $\Hop(\lambda;f) = \gtau^{-1} \Phi_t^\top \mathbf{M}_t(f) \Phi_t + \lambda I$ and $V_t(\lambda) =  \Phi_t^\top\Phi_t + \lambda I$, for the first inequality, we have:
    \begin{align}
        \Hop(\lambda;f) & =  \gtau^{-1} \Phi_t^\top \mathbf{M}_t(f) \Phi_t + \lambda I  \succeq  \gtau^{-1} \kappa_{\Xs}(f)^{-1} \Phi_t^\top \Phi_t + \lambda I \\  & \succeq \min\left\{1, \gtau^{-1} \kappa_{\Xs}(f)^{-1} \right\} \left(   \Phi_t^\top\Phi_t + \lambda I  \right)   = \min\left\{1, \gtau^{-1} \kappa_{\Xs}(f)^{-1} \right\} V_t(\lambda).
    \end{align}
    Similarly, for the other inequality, we have:
    \begin{align}
        \Hop(\lambda;f) & =  \gtau^{-1} \Phi_t^\top \mathbf{M}_t(f) \Phi_t + \lambda I \preceq  \gtau^{-1} R_{\dot{\mu}}  \Phi_t^\top \Phi_t + \lambda I  \\
        &  \preceq \max\left\{1,\gtau^{-1} R_{\dot{\mu}} \right\} \left( \Phi_t^\top \Phi_t + \lambda I  \right) = \max\left\{1,\gtau^{-1} R_{\dot{\mu}} \right\}  V_t(\lambda).
    \end{align}
\end{proof}

\begin{lemm}\label{lemma:relationAlphaG}
    Let $\mathcal{H}$ be an RKHS. Let $f,f' \in \mathcal{H}$, then, under Assumptions \ref{ass:boundedNorm} and \ref{ass:boundedKernel}, for every $f''\in\mathcal{H}$, it holds that:
    \begin{itemize}[leftmargin=12pt]
    \item $\| f - f'\|_{\Hop(\lambda;f')} \le (1+2R_sB{K}) \left( \left\| g_t(f) - g_t(f'')\right\|_{\Hop^{-1}(\lambda;f)} + \left\| g_t(f') - g_t(f'')\right\|_{\Hop^{-1}(\lambda;f')} \right) $;
    \item $\| f -f'\|_{V_t(\lambda)} \le {(1+2R_sB{K})}{\max\{1,\gtau\kappa_{\Xs}(f')\}} $ \\ $ \text{ } \qquad\qquad\qquad\qquad \cdot \left( \left\| g_t(f) - g_t(f'')\right\|_{V_t^{-1}(\lambda)} + \left\| g_t(f') - g_t(f'')\right\|_{V_t^{-1}(\lambda)} \right) $.
    \end{itemize}
    \end{lemm}
    
    \begin{proof}
        From the mean-value theorem (Equation~\ref{eq:mvt}), we have:
        \begin{align}
            g_t(f) - g_t(f') = G_t(f,f') (f- f').
        \end{align}
        The first statement follows the same derivation of Proposition 4 of \citep{abeille2021instance}, with the only care of applying Equation~\eqref{eq:7}:
         {\thinmuskip=1mu
\medmuskip=1mu
\thickmuskip=1mu
        \begin{align}
            \| f - f'\|_{\Hop(\lambda;f')} & \le  \sqrt{1+2R_sB{K}} \| f - f'\|_{G_t(f,f') } =  \sqrt{1+2R_sB{K}} \left\| g_t(f) - g_t(f')\right\|_{G_t^{-1}(f,f')} \label{eq:--1}\\ 
            & \le \sqrt{1+2R_sB{K}} \left( \left\| g_t(f) - g_t(f'')\right\|_{G_t^{-1}(f,f')} +  \left\| g_t(f') - g_t(f'')\right\|_{G_t^{-1}(f,f')} \right) \label{eq:--15}\\
            & \le (1+2R_sB{K}) \left( \left\| g_t(f) - g_t(f'')\right\|_{\widetilde{V}_t^{-1}(\lambda;f)} +  \left\| g_t(f') - g_t(f'')\right\|_{\widetilde{V}_t^{-1}(\lambda;f')} \right), \label{eq:--2}
        \end{align}}%
        where in lines \eqref{eq:--1} and \eqref{eq:--2}, we applied Equation~\eqref{eq:7}.
        The second statement starts from line \eqref{eq:--15}:
         {\thinmuskip=1mu
\medmuskip=1mu
\thickmuskip=1mu
        \begin{align}
        \| f -f'\|_{\Hop(\lambda;f')} & \le \sqrt{1+2R_sB{K}}\left( \left\| g_t(f) - g_t(f'')\right\|_{G_t^{-1}(f,f')} + \left\| g_t(f') - g_t(f'')\right\|_{G_t^{-1}(f,f')} \right) \\
        & \le (1+2R_sB{K})\left( \left\| g_t(f) - g_t(f'')\right\|_{\Hop^{-1}(\lambda;f')} + \left\| g_t(f') - g_t(f'')\right\|_{\Hop^{-1}(\lambda;f')} \right).
        \end{align}}%
        Then, we use the semidefinite inequality $\Hop(\lambda;f') \succeq \max\{1,\gtau \kappa_{\Xs}(f')\}^{-1} V_t(\lambda) $ (Lemma~\ref{lemma:VHbounds}):
        \begin{align}
        & \|f - f'\|_{\Hop(\lambda;f')} \ge \max\{1,\gtau \kappa_{\Xs}(f')\}^{-1/2}  \| f - f'\|_{V_t(\lambda)}\\
        &  \left\| g_t(\overline{f}) - g_t(f'')\right\|_{\Hop^{-1}(\lambda;f')} \le  \max\{1,\gtau \kappa_{\Xs}(f')\}^{1/2} \left\| g_t(\overline{f}) - g_t(f'')\right\|_{V_t^{-1}(\lambda)},
        \end{align}
        for $\overline{f} \in \{f,f'\}$.
    \end{proof}

\vspace{-.4cm}

\begin{lemm}\label{lemma:boundIG}
    Let $t\in \Nat$, $f \in \Hs$, $\mathbf{K}_t$ and $\widetilde{\mathbf{K}}_t(f)$ defined as in the main paper. It holds that:
    \begin{align}
       \log \det(\mathbf{I} + \lambda^{-1}\widetilde{\mathbf{K}}_t(f) ) & 
       \le \max\{1,R_{\dot{\mu}}\gtau^{-1}\}\log \det( \mathbf{I}+\lambda^{-1}\mathbf{K}_t) .
    \end{align}
\end{lemm}

\vspace{-.45cm}

\begin{proof}
    We can look at matrix $\widetilde{\mathbf{K}}_t(f)$ as $
        \widetilde{\mathbf{K}}_t(f) = \gtau^{-1} \mathbf{M}_t(f)^{1/2} \mathbf{K}_t  \mathbf{M}_t(f)^{1/2}$,
    where $ \mathbf{M}_t(f) = \mathrm{diag}((\dot{\mu}(f(\xs_s))_{s\in \dsb{t-1}}))$ is a diagonal positive definite matrix. $\widetilde{\mathbf{K}}_t(f)$ is symmetric positive definite, thus, using Horn's inequality for eigenvalues in multiplicative form~\citep{zhan2004matrix}, we have that for every $i \in \dsb{t-1}$:
    \begin{align}
        \lambda_i(\widetilde{\mathbf{K}}_t(f)) \le \lambda_i(\mathbf{K}_t) \lambda_{\max}( \mathbf{M}_t(f)) \gtau^{-1}  \le \lambda_i(\mathbf{K}_t) R_{\dot{\mu}}\gtau^{-1},
    \end{align}
    where $\lambda_i(\cdot)$ denotes the $i$-th eigenvalue.
    Furthermore, using Horn's inequality for eigenvalues in additive form, we have for $i \in \dsb{t-1}$:
    \begin{align}
        \lambda_i(\mathbf{I} + & \lambda^{-1}\widetilde{\mathbf{K}}_t(f))  = 1 +  \lambda^{-1}\lambda_i(\widetilde{\mathbf{K}}_t(f)) 
         \le 1+\lambda^{-1}R_{\dot{\mu}}\gtau^{-1}\lambda_i(\mathbf{K}_t) \\
        & \le 1+\lambda^{-1}\max\{1,R_{\dot{\mu}}\gtau^{-1}\}\lambda_i(\mathbf{K}_t)
         \le \left( 1+\lambda^{-1} \lambda_i(\mathbf{K}_t) \right)^{\max\{1,R_{\dot{\mu}}\gtau^{-1}\}} \\
        & = \lambda_i(\mathbf{I}+\lambda^{-1}\mathbf{K}_t)^{\max\{1,R_{\dot{\mu}}\gtau^{-1}\}},
    \end{align}
    where we exploited the inequality $1+ab \le (1+b)^a$ for $b\ge 0$ and $a \ge 1$.
    The statement is obtained passing to the determinant and to its logarithm.
\end{proof}

\vspace{-.45cm}

\begin{lemm}[Hilbert Elliptic Potential Lemma (slightly extended)]\label{lemma:elliptic}
     Let $(\psi_t)_{t \ge 1}$ be a sequence valued in the Hilbert space $\mathcal{H}$ and such that $\|\psi_t\| < +\infty$ for every $t \ge 1$, let $M \ge 1$, and $V_t(\lambda) = \sum_{s=1}^{t-1} \psi_s \psi_s^\top + \lambda I$ with $\lambda > 0$. For every $T \ge 2$, it holds that:
    \begin{align}
        \sum_{t=1}^{T-1} \min\{M, \|\psi_{t}\|_{V_{t}^{-1}(\lambda)}\}^2 \le 2 M \log \det(\lambda^{-1} V_T(\lambda)).
    \end{align}
\end{lemm}

\vspace{-.5cm}

\begin{proof}
    We observe that $V_t(\lambda)$ is invertible, self-adjoint and, for $\lambda > 0$, positive definite (thus,
    
    \noindent invertible). Using the inequality $\min\{1,u\} \le 2\log(1+u)$ for every $u \ge 0$, we have:

    \begin{align}
        \sum_{t=1}^{T-1} & \min\{M, \|\psi_{t}\|_{V_{t}^{-1}(\lambda)}^2\}  = M  \sum_{t=1}^{T-1} \min\{1, M^{-1}\|\psi_{t}\|_{V_{t}^{-1}(\lambda)}^2\}\\
        & \le  2M  \sum_{t=1}^{T-1} \log\left(1 + M^{-1}\|\psi_{t}\|_{V_{t}^{-1}(\lambda)}^2 \right) \le  2M  \sum_{t=1}^{T-1} \log\left(1 + \|\psi_{t}\|_{V_{t}^{-1}(\lambda)}^2 \right) , \label{eq:c6:1}
    \end{align}
    having exploited that $M \ge 1$. Now the last equation can be bounded following the usual steps of \citep{abeille2021instance}, with the only care of accounting for the fact that we are working in a Hilbert space. We have for $t \in \dsb{T}$:
    \begin{align}
        V_t(\lambda) = V_{t-1}(\lambda) + \psi_{t-1}\psi_{t-1}^\top = V_{t-1}^{1/2}(\lambda)\left(I + V_{t-1}^{-1/2} (\lambda)  \psi_{t-1}\psi_{t-1}^\top V_{t-1}^{-1/2}(\lambda)\right)V_{t-1}^{1/2}(\lambda).
    \end{align}
    From the Weinstein-Aronszajn identity \citep{kato2013perturbation}, being $V_{t-1}^{-1/2} (\lambda)  \psi_{t-1}\psi_{t-1}^\top V_{t-1}^{-1/2}(\lambda)$ of rank one and, thus, trace-class, we have:
        {\thinmuskip=1mu
\medmuskip=1mu
\thickmuskip=1mu
    \begin{align}
    \det& (\lambda^{-1}V_t(\lambda))  =  \det\left(\lambda^{-1} V_{t-1}^{1/2}(\lambda)\left(I + V_{t-1}^{-1/2} (\lambda)  \psi_{t-1}\psi_{t-1}^\top V_{t-1}^{-1/2}(\lambda)\right)V_{t-1}^{1/2}(\lambda)\right) \\
    & = \det(\lambda^{-1/2}V_{t-1}(\lambda)^{1/2}) \det\left(I + V_{t-1}^{-1/2} (\lambda)  \psi_{t-1}\psi_{t-1}^\top V_{t-1}^{-1/2}(\lambda)\right)  \det(\lambda^{-1/2}V_{t-1}(\lambda)^{1/2}) \\
    & = \det(\lambda^{-1}V_{t-1}(\lambda)) \left(1 + \| \psi_{t-1}\|^2_{V_{t-1}^{-1}(\lambda)}\right),
    \end{align}}%
    having observed that $\det\left(I + V_{t-1}^{-1/2} (\lambda)  \psi_{t-1}\psi_{t-1}^\top V_{t-1}^{-1/2}(\lambda)\right) = 1 + \| \psi_{t-1}\|^2_{V_{t-1}^{-1}(\lambda)} $.
    By unrolling and taking the logarithm, recalling that $\det(\lambda^{-1}V_1(\lambda)) = \det(I) = 1$, we have:
    \begin{align}
        \log \det(\lambda^{-1}V_T(\lambda))  =  \sum_{t=1}^{T-1} \log \left(1 + \| \psi_{t}\|^2_{V_{t}^{-1}(\lambda)}\right).
    \end{align}
    Plugging this equality into Equation~\eqref{eq:c6:1} concludes the proof.
\end{proof}

\vspace{-.4cm}

\begin{lemm}\label{lemma:tech}
    Let $t \in \Nat$, let $S_t$ and $\Hop(\lambda)$ defined as in Theorem~\ref{thr:selfNormalizedComplex}. Under the same assumptions of Theorem~\ref{thr:selfNormalizedComplex}, the following inequalities hold almost surely:
    \begin{align}
        \|S_t\|_{\Hop^{-1}(\lambda)}^2 \le \frac{(t-1)^2K^2R^2}{\lambda}, \qquad \log \det (\lambda^{-1}\Hop(\lambda)) \le (t-1) \log\left( 1 +\frac{K^2R^2}{\lambda} \right).
    \end{align}
\end{lemm}

\vspace{-.4cm}

\begin{proof}
    For the first inequality, we proceed as follows:
    \begin{align}
        \|S_t\|_{\Hop^{-1}(\lambda)}^2 \le \|S_t\|_{\lambda^{-1} I}^2 \le \lambda^{-1} \left(\sum_{s=1}^{t-1} |\epsilon_s| \|\psi_s\| \right)^2 \le \frac{(t-1)^2K^2R^2}{\lambda},
    \end{align}
    having exploited that $\Hop(\lambda) \succeq \lambda I$, $|\epsilon_s| \le R$ a.s., and $\|\psi_s\| \le K$ for every $s \ge 1$.
    Let us denote with $\Psi_t = (\psi_1,\dots,\psi_{t-1})^\top$. For the second inequality, we proceed as follows:
    \begin{align}
        \det & (\lambda^{-1} \Hop(\lambda))  = \det(I + \lambda^{-1}\Psi_t^\top \Psi_t)   = \det(\mathbf{I} + \lambda^{-1}\Psi_t \Psi_t^\top) \\
        & \le  \left(\frac{1}{t-1} \mathrm{tr}(\mathbf{I} + \lambda^{-1} \Psi_t \Psi_t^\top ) \right)^{\!\!t-1} \!\!\!= \!  \left(1 + \frac{\lambda^{-1}}{t-1}\sum_{s=1}^{t-1} \sigma_s^2\|\psi_s\|^2  \right)^{\!\!t-1} \!\!\! \le \! \left(1 \!+ \!\frac{K^2R^2}{\lambda^2} \right)^{\!\!t-1}\!\!\!,
    \end{align}
    having applied the Weinstein-Aronszajn identity \citep{kato2013perturbation}: $\det(I + \lambda^{-1}\Psi_t^\top \Psi_t) = \det(\mathbf{I} + \lambda^{-1}\Psi_t \Psi_t^\top)$ being $\Psi_t^\top \Psi_t$ a trace-class operator, the determinant-trace inequality $\det(\mathbf{A}) \le \left( \mathrm{tr}(\mathbf{A})/d \right)^{d}$ for a symmetric positive semidefinite $d\times d$ matrix $\mathbf{A}$, and bounded $\sigma_s\|\psi_s\|^2 \le R K^2$, being the standard deviation $\sigma_s$ bounded by the range $R$.
\end{proof}


\bibliography{biblio}

@article{bubeck2015convex,
  title={Convex optimization: Algorithms and complexity},
  author={Bubeck, S{\'e}bastien},
  journal={Foundations and Trends in Machine Learning},
  volume={8},
  number={3-4},
  pages={231--357},
  year={2015},
  publisher={Now Publishers, Inc.}
}

@inproceedings{zhou2020neural,
  title={Neural contextual bandits with {UCB}-based exploration},
  author={Zhou, Dongruo and Li, Lihong and Gu, Quanquan},
  booktitle={International Conference on Machine Learning (ICML)},
  pages={11492--11502},
  volume={119},
  year={2020},
  organization={PMLR}
}

@inproceedings{zhang2021neural,
  title={Neural {T}hompson Sampling},
  author={Zhang, Weitong and Zhou, Dongruo and Li, Lihong and Gu, Quanquan},
  booktitle={International Conference on Learning Representation (ICLR)},
  year={2021}
}

@inproceedings{valko2013finite,
  title={Finite-Time Analysis of Kernelised Contextual Bandits},
  author={Valko, Michal and Korda, Nathan and Munos, R{\'e}mi and Flaounas, Ilias and Cristianini, Nello},
  booktitle={Uncertainty in Artificial Intelligence (UAI)},
  pages={654--663},
  year={2013},
  organization={AUAI}
}

@inproceedings{yang2020reinforcement,
  title={Reinforcement learning in feature space: Matrix bandit, kernels, and regret bound},
  author={Yang, Lin and Wang, Mengdi},
  booktitle={International Conference on Machine Learning (ICML)},
  pages={10746--10756},
  volume={119},
  year={2020},
  organization={PMLR}
}

@inproceedings{metelli2025a,
title={A Novel Self-Normalized {B}ernstein-Like Dimension-Free Inequality and Regret Bounds for Generalized Kernelized Bandits},
author={Alberto M. Metelli and Simone Drago and Marco Mussi},
booktitle={European Workshop on Reinforcement Learning},
year={2025}
}

@article{li2024q,
  title={Is {Q}-learning minimax optimal? {A} tight sample complexity analysis},
  author={Li, Gen and Cai, Changxiao and Chen, Yuxin and Wei, Yuting and Chi, Yuejie},
  journal={Operations Research},
  volume={72},
  number={1},
  pages={222--236},
  year={2024},
  publisher={INFORMS}
}

@inproceedings{lee2020bias,
  title={Bias no more: high-probability data-dependent regret bounds for adversarial bandits and mdps},
  author={Lee, Chung-Wei and Luo, Haipeng and Wei, Chen-Yu and Zhang, Mengxiao},
  booktitle={Advances in Neural Information Processing Systems (NeurIPS)},
  volume={33},
  pages={15522--15533},
  year={2020},
  publisher={Curran Associates, Inc.}
}

@inproceedings{zimmert2022return,
  title={Return of the bias: Almost minimax optimal high probability bounds for adversarial linear bandits},
  author={Zimmert, Julian and Lattimore, Tor},
  booktitle={Annual Conference on Learning Theory (COLT)},
  pages={3285--3312},
  volume={178},
  year={2022},
  organization={PMLR}
}

@inproceedings{rakhlin2012making,
  title={Making gradient descent optimal for strongly convex stochastic optimization},
  author={Rakhlin, Alexander and Shamir, Ohad and Sridharan, Karthik},
  booktitle={International Conference on Machine Learning (ICML)},
  pages={1571--1578},
  year={2012},
  publisher={Omnipress}
}

@book{kato2013perturbation,
  title={Perturbation theory for linear operators},
  author={Kato, Tosio},
  volume={132},
  year={2013},
  publisher={Springer}
}

@article{li2022sherman,
  title={The {S}herman-{M}orrison-{W}oodbury Formula for the Generalized Inverses},
  author={Li, Tingting and Mosi{\'c}, Dijana and Chen, Jianlong},
  journal={Filomat},
  volume={36},
  number={15},
  pages={5307--5313},
  year={2022},
  publisher={JSTOR}
}

@book{gohberg2012traces,
  title={Traces and determinants of linear operators},
  author={Gohberg, Israel and Goldberg, Seymour and Krupnik, Nahum},
  volume={116},
  year={2012},
  publisher={Birkh{\"a}user}
}

@inproceedings{ZhaoHZZG23,
  author       = {Heyang Zhao and
                  Jiafan He and
                  Dongruo Zhou and
                  Tong Zhang and
                  Quanquan Gu},
  title        = {Variance-Dependent Regret Bounds for Linear Bandits and Reinforcement Learning: Adaptivity and Computational Efficiency},
  booktitle    = {Annual Conference on Learning Theory (COLT)},
  volume       = {195},
  pages        = {4977--5020},
  publisher    = {{PMLR}},
  year         = {2023}
}

@article{mussi2024open,
  title={Open Problem: Tight Bounds for Kernelized Multi-Armed Bandits with Bernoulli Rewards},
  author={Mussi, Marco and Drago, Simone and Metelli, Alberto M.},
  journal={arXiv preprint arXiv:2407.06321},
  year={2024}
}

@article{bae2025neural,
  title={Neural Logistic Bandits},
  author={Bae, Seoungbin and Lee, Dabeen},
  journal={arXiv preprint arXiv:2505.02069},
  year={2025}
}

@article{akhavan2025bernstein,
  title={Bernstein-type dimension-free concentration for self-normalised martingales},
  author={Akhavan, Arya and Shidani, Amitis and Ayoub, Alex and Janz, David},
  journal={arXiv preprint arXiv:2507.20982},
  year={2025}
}

@inproceedings{GaillardSE14,
  author       = {Pierre Gaillard and
                  Gilles Stoltz and
                  Tim van Erven},
  title        = {A second-order bound with excess losses},
  booktitle    = {Annual Conference on Learning Theory (COLT)},
  volume       = {35},
  pages        = {176--196},
  publisher    = {PMLR},
  year         = {2014},
}

@inproceedings{VakiliKP21,
  author       = {Sattar Vakili and
                  Kia Khezeli and
                  Victor Picheny},
  title        = {On Information Gain and Regret Bounds in Gaussian Process Bandits},
  booktitle    = {International Conference on Artificial Intelligence and Statistics (AISTATS)},
  volume       = {130},
  pages        = {82--90},
  publisher    = {{PMLR}},
  year         = {2021},
}

@article{bubeck2011x,
  title={X-Armed Bandits.},
  author={Bubeck, S{\'e}bastien and Munos, R{\'e}mi and Stoltz, Gilles and Szepesv{\'a}ri, Csaba},
  journal={Journal of Machine Learning Research},
  volume={12},
  number={5},
  year={2011},
  pages={1655-1695}
}

@book{zhan2004matrix,
  title={Matrix inequalities},
  author={Zhan, Xingzhi},
  year={2004},
  publisher={Springer}
}

@inproceedings{rando2022ada,
  title={Ada-{BKB}: Scalable Gaussian process optimization on continuous domains by adaptive discretization},
  author={Rando, Marco and Carratino, Luigi and Villa, Silvia and Rosasco, Lorenzo},
  booktitle={International Conference on Artificial Intelligence and Statistics (AISTATS)},
  pages={7320--7348},
  volume={151},
  year={2022},
  organization={PMLR}
}

@phdthesis{whitehouse2024modern,
  title={Modern Martingale Methods: Theory and Applications},
  author={Whitehouse, Justin A.},
  year={2024},
  school={Carnegie Mellon University}
}

@phdthesis{abbasi2013online,
  title={Online learning for linearly parametrized control problems},
  author={Abbasi-Yadkori, Yasin},
  year={2013},
  school={University of Alberta}
}

@article{ziemann2024vector,
title={A Vector {B}ernstein Inequality for Self-Normalized Martingales},
author={Ingvar Ziemann},
journal={Transactions on Machine Learning Research},
year={2025}
}

@book{woodbury1950inverting,
  title={Inverting modified matrices},
  author={Woodbury, Max A.},
  year={1950},
  publisher={Princeton University}
}

@inproceedings{zhou2021nearly,
  title={Nearly minimax optimal reinforcement learning for linear mixture {M}arkov decision processes},
  author={Zhou, Dongruo and Gu, Quanquan and Szepesv{\'a}ri, Csaba},
  booktitle={Annual Conference on Learning Theory (COLT)},
  pages={4532--4576},
  volume={134},
  year={2021},
  organization={PMLR}
}

@article{freedman1975tail,
  title={On tail probabilities for Martingales},
  author={Freedman, David A.},
  journal={The Annals of Probability},
  pages={100--118},
  year={1975},
  volume={3},
  number={1},
  publisher={JSTOR}
}

@article{howard2021time,
  title={Time-uniform, nonparametric, nonasymptotic confidence sequences},
  author={Howard, Steven R. and Ramdas, Aaditya and McAuliffe, Jon and Sekhon, Jasjeet},
  journal={The Annals of Statistics},
  volume={49},
  number={2},
  pages={1055--1080},
  year={2021},
  publisher={JSTOR}
}

@book{zhang2023mathematical,
  title={Mathematical analysis of machine learning algorithms},
  author={Zhang, Tong},
  year={2023},
  publisher={Cambridge University Press}
}

@inproceedings{faury2020improved,
  title={Improved optimistic algorithms for logistic bandits},
  author={Faury, Louis and Abeille, Marc and Calauz{\`e}nes, Cl{\'e}ment and Fercoq, Olivier},
  booktitle={International Conference on Machine Learning (ICML)},
  pages={3052--3060},
  volume={119},
  year={2020},
  organization={PMLR}
}

@inproceedings{scholkopf2001generalized,
  title={A generalized representer theorem},
  author={Sch{\"o}lkopf, Bernhard and Herbrich, Ralf and Smola, Alex J.},
  booktitle={International Conference on Computational Learning Theory},
  pages={416--426},
  volume={2111},
  year={2001},
  organization={Springer}
}

@inproceedings{russac2021self,
  title={Self-concordant analysis of generalized linear bandits with forgetting},
  author={Russac, Yoan and Faury, Louis and Capp{\'e}, Olivier and Garivier, Aur{\'e}lien},
  booktitle={International Conference on Artificial Intelligence and Statistics (AISTATS)},
  pages={658--666},
  volume={130},
  year={2021},
  organization={PMLR}
}

@inproceedings{li2017provably,
  title={Provably optimal algorithms for generalized linear contextual bandits},
  author={Li, Lihong and Lu, Yu and Zhou, Dengyong},
  booktitle={International Conference on Machine Learning (ICML)},
  pages={2071--2080},
  volume={70},
  year={2017},
  organization={PMLR}
}

@inproceedings{srinivas2010gaussian,
  title={Gaussian Process Optimization in the Bandit Setting: No Regret and Experimental Design},
  author={Srinivas, Niranjan and Krause, Andreas and Kakade, Sham and Seeger, Matthias},
  booktitle={International Conference on Machine Learning (ICML)},
  pages={1015--1022},
  year={2010},
  organization={Omnipress}
}

@inproceedings{chowdhury2017kernelized,
  title={On kernelized multi-armed bandits},
  author={Chowdhury, Sayak R. and Gopalan, Aditya},
  booktitle={International Conference on Machine Learning (ICML)},
  pages={844--853},
  volume={70},
  year={2017},
  organization={PMLR}
}

@inproceedings{abeille2021instance,
  title={Instance-wise minimax-optimal algorithms for logistic bandits},
  author={Abeille, Marc and Faury, Louis and Calauz{\`e}nes, Cl{\'e}ment},
  booktitle={International Conference on Artificial Intelligence and Statistics (AISTATS)},
  pages={3691--3699},
  volume={130},
  year={2021},
  organization={PMLR}
}

@inproceedings{sawarnigeneralized,
  title={Generalized Linear Bandits with Limited Adaptivity},
  author={Sawarni, Ayush and Das, Nirjhar and Barman, Siddharth and Sinha, Gaurav},
  booktitle={Advances in Neural Information Processing Systems (NeurIPS)},
  year={2024},
  volume={267},
  pages={8329--8369}
}

@inproceedings{lee2024unified,
  title={A Unified Confidence Sequence for Generalized Linear Models, with Applications to Bandits},
  author={Lee, Junghyun and Yun, Se-Young and Jun, Kwang-Sung},
  booktitle={Advances in Neural Information Processing Systems (NeurIPS)},
  year={2024},
  volume={37},
  pages={124640--124685}
}

@book{barndorff2014information,
  title={Information and exponential families in statistical theory},
  author={Barndorff-Nielsen, Ole},
  year={1978},
  publisher={John Wiley \& Sons}
}

@book{smola1998learning,
  title={Learning with kernels},
  author={Smola, Alexander J. and Sch{\"o}lkopf, Bernhard},
  volume={4},
  year={1998},
  publisher={GMD-Forschungszentrum Informationstechnik}
}

@book{mccullagh2019generalized,
  title={Generalized linear models},
  author={Mccullagh, Peter and Nelder, John A.},
  year={1989},
  publisher={Chapman and Hall}
}

@inproceedings{abbasiyadkori2011improved,
  author       = {Yasin Abbasi{-}Yadkori and
                  D{\'{a}}vid P{\'{a}}l and
                  Csaba Szepesv{\'{a}}ri},
  title        = {Improved Algorithms for Linear Stochastic Bandits},
  booktitle    = {Advances in Neural Information Processing Systems (NIPS)},
  pages        = {2312--2320},
  year         = {2011}
}

@inproceedings{dani2008stochastic,
  author       = {Varsha Dani and
                  Thomas P. Hayes and
                  Sham M. Kakade},
  title        = {Stochastic Linear Optimization under Bandit Feedback},
  booktitle    = {Annual Conference on Learning Theory (COLT)},
  pages        = {355--366},
  publisher    = {Omnipress},
  year         = {2008}
}

@article{delapena2004self,
  title={Self-normalized processes: exponential inequalities, moment bounds and iterated logarithm laws},
  author={de la Pe{\~n}a, Victor H. and Klass, Michael J. and Lai, Tze L.},
  journal={The Annals of Probability},
  pages={1902--1933},
  volume={32},
  number={3},
  year={2004}
}

@article{lai1985asymptotically,
  title={Asymptotically efficient adaptive allocation rules},
  author={Lai, Tze L. and Robbins, Herbert},
  journal={Advances in applied mathematics},
  volume={6},
  number={1},
  pages={4--22},
  year={1985},
  publisher={Academic Press}
}

@inproceedings{filippi2010parametric,
  author       = {Sarah Filippi and
                  Olivier Capp{\'{e}} and
                  Aur{\'{e}}lien Garivier and
                  Csaba Szepesv{\'{a}}ri},
  title        = {Parametric Bandits: The Generalized Linear Case},
  booktitle    = {Advances in Neural Information Processing Systems (NIPS)},
  pages        = {586--594},
  publisher    = {Curran Associates, Inc.},
  year         = {2010}
}

@book{rasmussen2006gaussian,
  author       = {Carl E. Rasmussen and
                  Christopher K. I. Williams},
  title        = {Gaussian processes for machine learning},
  publisher    = {{MIT} Press},
  year         = {2006}
}

@inproceedings{whitehouse2023sublinear,
  author       = {Justin A. Whitehouse and
                  Aaditya Ramdas and
                  Zhiwei S. Wu},
  title        = {On the Sublinear Regret of {GP-UCB}},
  booktitle    = {Advances in Neural Information Processing Systems (NeurIPS)},
  volume       = {36},
  pages        = {35266--35276},
  year         = {2023}
}

@inproceedings{scarlett2017lower,
  author       = {Jonathan Scarlett and
                  Ilija Bogunovic and
                  Volkan Cevher},
  title        = {Lower Bounds on Regret for Noisy Gaussian Process Bandit Optimization},
  booktitle    = {Annual Conference on Learning Theory (COLT)},
  volume       = {65},
  pages        = {1723--1742},
  publisher    = {{PMLR}},
  year         = {2017}
}

\end{document}